\newif\ifisarxiv
 \isarxivtrue

 \ifisarxiv
 \documentclass[11pt]{article}
 \usepackage[round]{natbib}
 \usepackage{fullpage}
 \usepackage{amsmath}
 \usepackage{amssymb}
 \usepackage{amsthm}
 \usepackage{mathabx}
 \usepackage{color}
 \usepackage{cancel}
 \usepackage{wrapfig}
 \usepackage{subfigure}
 \usepackage{caption}
 \usepackage{xfrac}
 \usepackage{algorithm}
 \usepackage{algorithmic}
 \usepackage{mdframed}
 \usepackage{hyperref}
 \usepackage{xcolor}
 \hypersetup{
 	colorlinks,
 	linkcolor={red!40!gray},
 	citecolor={blue!40!gray},
 	urlcolor={blue!70!gray}
 }
 \usepackage{cleveref}

 \title{Exact expressions for double descent
 and implicit regularization
 \\
 via surrogate random design}
   \author{%
           \textbf{Micha{\l } Derezi\'{n}ski} \\
   Department of Statistics\\
   University of California, Berkeley\\
   \texttt{mderezin@berkeley.edu}\\
   \and
   \textbf{Feynman Liang} \\
   Department of Statistics\\
   University of California, Berkeley\\
   \texttt{feynman@berkeley.edu}
   \and
    \textbf{Michael W. Mahoney}\\
   ICSI and Department of Statistics\\
   University of California, Berkeley\\
   \texttt{mmahoney@stat.berkeley.edu}
 }

   \providecommand{\BlackBox}{\rule{1.5ex}{1.5ex}}
   \renewenvironment{proof}%
   {%
    \par\noindent{\bfseries\upshape Proof\ }%
   }%
 {\hfill\BlackBox\par\bigskip}

\fi

\newcommand{\pdet}{{\mathrm{pdet}}}

\newcommand{\MSE}[1] {{\mathrm{MSE}\big[#1\big]}}
\def\Poisson{{\operatorname{Poisson}}}

\def\Ic{\mathcal{I}}
\def\Jc{\mathcal{J}}
\def\Mc{\mathcal M}

\def\ktd{{k^{\underline{d}}}}
\def\Det{{\mathrm{Det}}}

\newif\ifDRAFT
\DRAFTtrue
\ifDRAFT
\newcommand{\marrow}{\marginpar[\hfill$\longrightarrow$]{$\longleftarrow$}}
\newcommand{\niceremark}[3]
   {\textcolor{red}{\textsc{#1 #2:} \marrow\textsf{#3}}}
\newcommand{\ken}[2][says]{\niceremark{Ken}{#1}{#2}}

\newcommand{\michael}[2][says]{\niceremark{Michael}{#1}{#2}}
\newcommand{\michal}[2][says]{\niceremark{Michal}{#1}{#2}}
\newcommand{\feynman}[2][says]{\niceremark{Feynman}{#1}{#2}}
\else
\newcommand{\ken}[1]{}
\newcommand{\michael}[1]{}
\newcommand{\michal}[1]{}
\newcommand{\feynman}[1]{}
\fi

\def\ee{\mathrm{e}}

\newenvironment{proofof}[2]{\par\vspace{2mm}\noindent\textbf{Proof of {#1} {#2}}\ }{\hfill\BlackBox}

\DeclareMathOperator{\adj}{\textnormal{adj}}

\def\xib{\boldsymbol\xi}
\def\Sigmab{\mathbf{\Sigma}}

\def\S{\mathbf{S}}

\def\Xt{\widetilde{X}}

\def\Xb{{\bar{\X}}}
\def\ybb{\overline{\y}}
\def\f{{\mathbf{f}}}

\def\W{\mathbf W}

\def\Pc{\mathcal{P}}
\def\Nc{\mathcal{N}}

\def\Wc{\mathcal{W}}

\def\Q{\mathbf Q}

\ifx\BlackBox\undefined
\newcommand{\BlackBox}{\rule{1.5ex}{1.5ex}}  
\fi
\DeclareMathOperator*{\argmin}{\mathop{\mathrm{argmin}}}

\def\x{\mathbf x}
\def\y{\mathbf y}

\def\ybb{\bar{\mathbf y}}
\def\xbb{\bar{\mathbf x}}
\def\yb{{\bar y}}

\def\z{\mathbf z}

\def\w{\mathbf w}
\def\v{\mathbf v}

\def\wbh{\widehat{\mathbf w}}

\def\e{\mathbf e}
\def\zero{\mathbf 0}
\def\one{\mathbf 1}
\def\u{\mathbf u}

\def\f{\mathbf f}

\def\X{\mathbf X}

\def\B{\mathbf B}
\def\A{\mathbf A}

\def\D{\mathbf D}
\def\V{\mathbf V}

\def\Vc{\mathcal{V}}
\def\Bc{\mathcal{B}}

\def\Z{\mathbf Z}

\def\I{\mathbf I}
\def\Ic{\mathcal I}

\def\A{\mathbf A}

\def\Xt{\widetilde{\mathbf{X}}}

\def\E{\mathbb E}
\def\R{\mathbb R}

\def\Pr{\mathrm{Pr}}
\def\tr{\mathrm{tr}}

\def\rank{\mathrm{rank}}

\def\Var{\mathrm{Var}}

\def\pdet{\mathrm{pdet}}

\let\origtop\top
\renewcommand\top{{\scriptscriptstyle{\origtop}}} 

\definecolor{silver}{cmyk}{0,0,0,0.3}
\definecolor{yellow}{cmyk}{0,0,0.9,0.0}
\definecolor{reddishyellow}{cmyk}{0,0.22,1.0,0.0}
\definecolor{black}{cmyk}{0,0,0.0,1.0}
\definecolor{darkYellow}{cmyk}{0.2,0.4,1.0,0}
\definecolor{orange}{cmyk}{0.0,0.7,0.9,0}
\definecolor{darkSilver}{cmyk}{0,0,0,0.1}
\definecolor{grey}{cmyk}{0,0,0,0.5}
\definecolor{darkgreen}{cmyk}{0.6,0,0.8,0}

\ifx\proof\undefined
\newenvironment{proof}{\par\noindent{\bf Proof\ }}{\hfill\BlackBox\\[2mm]}
\fi

\ifx\theorem\undefined
\newtheorem{theorem}{Theorem}
\fi

\ifx\example\undefined
\newtheorem{example}{Example}
\fi

\ifx\condition\undefined
\newtheorem{condition}{Condition}
\fi
\ifx\property\undefined
\newtheorem{property}{Property}
\fi

\ifx\lemma\undefined
\newtheorem{lemma}{Lemma}
\fi

\ifx\proposition\undefined
\newtheorem{proposition}{Proposition}
\fi

\ifx\remark\undefined
\newtheorem{remark}{Remark}
\fi

\ifx\corollary\undefined
\newtheorem{corollary}{Corollary}
\fi

\ifx\definition\undefined
\newtheorem{definition}{Definition}
\fi

\ifx\conjecture\undefined
\newtheorem{conjecture}{Conjecture}
\fi

\ifx\axiom\undefined
\newtheorem{axiom}{Axiom}
\fi

\ifx\claim\undefined
\newtheorem{claim}{Claim}
\fi

\ifx\assumption\undefined
\newtheorem{assumption}{Assumption}
\fi

\ifx\condition\undefined
\newtheorem{condition}{Condition}
\fi

\begin{document}
\maketitle

\begin{abstract}
Double descent refers to the phase transition that is exhibited by
the generalization error of unregularized learning models when varying the ratio
between the number of parameters and the number of training
samples. The recent success of highly over-parameterized machine learning
models such as deep neural networks has motivated a theoretical analysis of
the double descent phenomenon in classical models such as linear
regression which can also generalize well in the over-parameterized
regime. We provide the first exact non-asymptotic
expressions for double descent of the minimum norm linear
estimator. Our approach involves constructing a special
determinantal point process  which we call surrogate random
design, to replace the standard i.i.d.~design of the training
sample. This surrogate design admits exact expressions for the mean
squared error of the estimator while preserving the key properties
of the standard design. We also establish an exact implicit
regularization result for over-parameterized training samples. In
particular, we show that, for the surrogate design, the implicit bias
of the unregularized minimum norm estimator precisely corresponds to
solving a ridge-regularized least squares problem on the population
distribution. In our analysis we introduce a new mathematical tool of
independent interest: the class of random matrices for which
determinant commutes with expectation. 
\end{abstract}

\section{Introduction}

Classical statistical learning theory asserts that to achieve generalization
one must use training sample size that sufficiently exceeds the complexity of
the learning model, where the latter is typically represented by the number of
parameters \citep[or some related structural parameter; see][]{HFT09}.  In particular,
this seems to suggest the conventional wisdom that one should not use models
that fit the training data exactly.  However, modern machine learning practice
often seems to go against this intuition, using models with so many parameters
that the training data can be perfectly interpolated, in which case the
training error vanishes. It has been shown that models such as deep neural
networks, as well as certain so-called interpolating kernels and decision
trees, can generalize well in this regime. In particular,
\cite{BHMM19} empirically demonstrated a phase transition in generalization
performance of learning models which occurs at an \emph{interpolation
thershold}, i.e., a point where training error goes to zero (as one varies the
ratio between the model complexity and the sample size). Moving away from this
threshold in either direction tends to reduce the generalization error, leading
to the so-called \emph{double descent} curve.

To understand this
surprising phenomenon, in perhaps the simplest possible setting, we
study it in the context of linear or least squares regression.
Consider a full rank $n\times d$ data matrix $\X$ and a vector $\y$ of
responses corresponding to each of the $n$ data points (the rows of $\X$), where we wish to
find the best linear model $\X\w\approx \y$, parameterized by a
$d$-dimensional vector $\w$.
The simplest example of an estimator that has been shown to exhibit
the double descent phenomenon \citep{belkin2019two} is the
Moore-Penrose estimator, $\wbh=\X^\dagger\y$:
in the so-called over-determined regime, i.e., when $n>d$, it corresponds to the
least squares solution, i.e., $\argmin_{\w} \|\X\w-\y\|^2$; and in the
under-determined regime (also known as
over-parameterized or interpolating), i.e., when $n<d$, it
corresponds to the minimum norm solution to the linear system $\X\w=\y$.
Given the ubiquity of linear regression and the Moore-Penrose
solution, e.g., in kernel-based machine learning, studying the
performance of this estimator can shed some light on the effects of
over-parameterization/interpolation in machine learning more generally.
Of particular interest are results that are exact (i.e., not upper/lower bounds) and
non-asymptotic (i.e., for large but still finite $n$ and~$d$).

We build on methods from Randomized Numerical Linear Algebra (RandNLA) in order
to obtain \emph{exact non-asymptotic expressions} for the mean squared error
(MSE) of the Moore-Penrose estimator (see Theorem~\ref{t:mse}).  This provides
a precise characterization of the double descent phenomenon for the
linear regression problem. In obtaining these results,
we are able to provide precise formulas for the \emph{implicit regularization}
induced by minimum norm solutions of under-determined training samples,
relating it to classical ridge regularization (see Theorem~\ref{t:unbiased}).
To obtain our precise results, we use a somewhat non-standard random
design, based on a specially chosen determinantal point process (DPP), which we
term surrogate random design. DPPs are a family of non-i.i.d.~sampling
distributions which are typically used to induce diversity in the
produced samples \citep{dpp-ml}. Our aim in using a DPP as a surrogate
design is very different: namely, to make certain quantities (such as the MSE)
analytically tractable, while accurately \emph{preserving} the underlying
properties of the original data distribution. This strategy might seem
counter-intuitive since DPPs are typically found most useful when they
\emph{differ} from the data distribution. However, we show both theoretically
(Theorem~\ref{t:asymptotic}) and empirically
(Section~\ref{sec:asymp-conj-details}), that for many commonly studied data
distributions, such as multivariate Gaussians, our DPP-based surrogate
design accurately preserves the key properties of the standard i.i.d.~design
(such as the MSE), and even matches it exactly in the high-dimensional
asymptotic limit. In our analysis of the surrogate design, we
introduce the concept of \emph{determinant preserving
 random matrices} (Section \ref{s:dp}), a class of random matrices for which determinant
commutes with expectation, which should be of independent interest.

\subsection{Main results: double descent and implicit regularization}

As the performance metric in our analysis, we use the \emph{mean
  squared error} (MSE), defined as
$\mathrm{MSE}[\wbh]=\E\big[\|\wbh-\w^*\|^2\big]$, where $\w^*$ is a fixed
underlying linear model of the responses.
In analyzing the MSE, we make the following standard assumption that
the response noise is homoscedastic.
\begin{assumption}[Homoscedastic noise]\label{a:linear}
The noise $\xi =y(\x)- \x^\top\w^*$ has mean $0$ and variance $\sigma^2$.
\end{assumption}

\noindent
Our main result provides an exact expression for the MSE of the
Moore-Penrose estimator under our surrogate design denoted $\Xb\sim S_\mu^n$, where
$\mu$ is the $d$-variate distribution of the row vector $\x^\top$ and $n$ is the sample
size. This surrogate is
used in place of the standard $n\times d$ random design $\X\sim\mu^n$, where $n$
data points (the rows of $\X$) are sampled independently from
$\mu$. We form the surrogate by constructing a determinantal point
process with $\mu$ as the background measure, so that $S_\mu^n(\X)\propto
\pdet(\X\X^\top)\mu(\X)$, where $\pdet(\cdot)$ denotes the pseudo-determinant
(details in Section~\ref{s:determinantal}).  Unlike for the standard
design, our MSE formula is fully expressible as a function of the covariance
matrix $\Sigmab_\mu=\E_\mu[\x\x^\top]$. To state our main result, we
need an additional minor assumption on $\mu$ which is satisfied by
most standard continuous distributions (e.g., multivariate Gaussians).
\begin{assumption}[General position]\label{a:general-position}
For $1\leq n \leq d$, if $\X\sim\mu^n$, then $\rank(\X)=n$ almost surely.
\end{assumption}

\noindent
Under Assumptions~\ref{a:linear} and~\ref{a:general-position}, we can
establish our first main result, stated as the following theorem, where
we use $\X^\dagger$ to denote the Moore-Penrose inverse of $\X$.

\begin{theorem}[Exact non-asymptotic MSE]
\label{t:mse}
  If the response noise is homoscedastic 
  (Assumption~\ref{a:linear}) and $\mu$ is in
  general position (Assumption~\ref{a:general-position}), then for
  $\Xb\sim S_\mu^n$ (Definition \ref{d:surrogate}) and
  $\yb_i=y(\xbb_i)$, 
\vspace{-2mm}
  \begin{align*}
    \MSE{\Xb^\dagger\ybb} =
    \begin{cases}
    \sigma^2\,\tr\big((\Sigmab_\mu+\lambda_n\I)^{-1}\big)\cdot
    \frac{1-\alpha_n}{d-n}\ +\
\frac{\w^{*\top}(\Sigmab_\mu+\lambda_n\I)^{-1}\w^*}
{\tr((\Sigmab_\mu+\lambda_n\I)^{-1})}\cdot (d-n),
&\text{for }n<d,\\
\sigma^2\, \tr(\Sigmab_\mu^{-1}),& \text{for }n=d,\\
\sigma^2\,\tr(\Sigmab_\mu^{-1})\cdot\frac{1-\beta_n}{n-d},&\text{for
}n>d,
\end{cases}
  \end{align*}  
with $\lambda_n\geq 0$ defined by
  $n=\tr(\Sigmab_\mu (\Sigmab_\mu+\lambda_n\I)^{-1})$, \
  $\alpha_n=\det(\Sigmab_\mu(\Sigmab_\mu+\lambda_n\I)^{-1})$
  and $\beta_n=\ee^{d-n}$.
\end{theorem}
\begin{definition}\label{d:expression}
  We will use $\Mc = \Mc(\Sigmab_\mu, \w^*,\sigma^2,n)$ to denote the above expressions
  for $\MSE{\Xb^\dagger\ybb}$.
\end{definition}

\begin{figure}[t]
\centering
\subfigure[%
  Surrogate MSE expressions (Theorem \ref{t:mse}) closely match
  numerical estimates even for non-isotropic
  features. Eigenvalue decay leads to a steeper
  descent curve in the under-determined regime ($n<d$).]{%
    \includegraphics[width=0.48\textwidth]{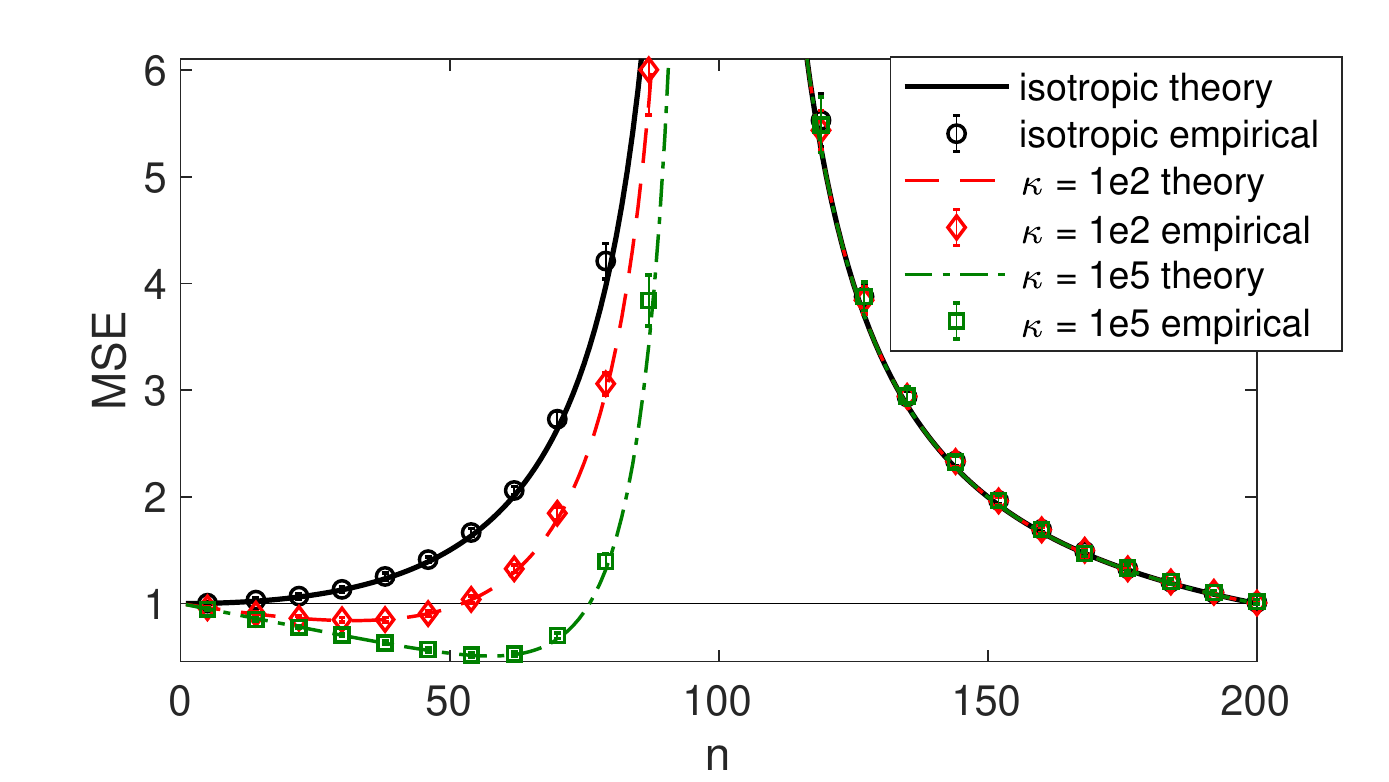}
  }
\hfill
\subfigure[%
    The mean of the estimator $\X^\dagger\y$ exhibits
    shrinkage which closely matches the shrinkage of a
    ridge-regularized least squares optimum (theory lines), as characterized by
    Theorem \ref{t:unbiased}.]{%
      \includegraphics[width=0.48\textwidth]{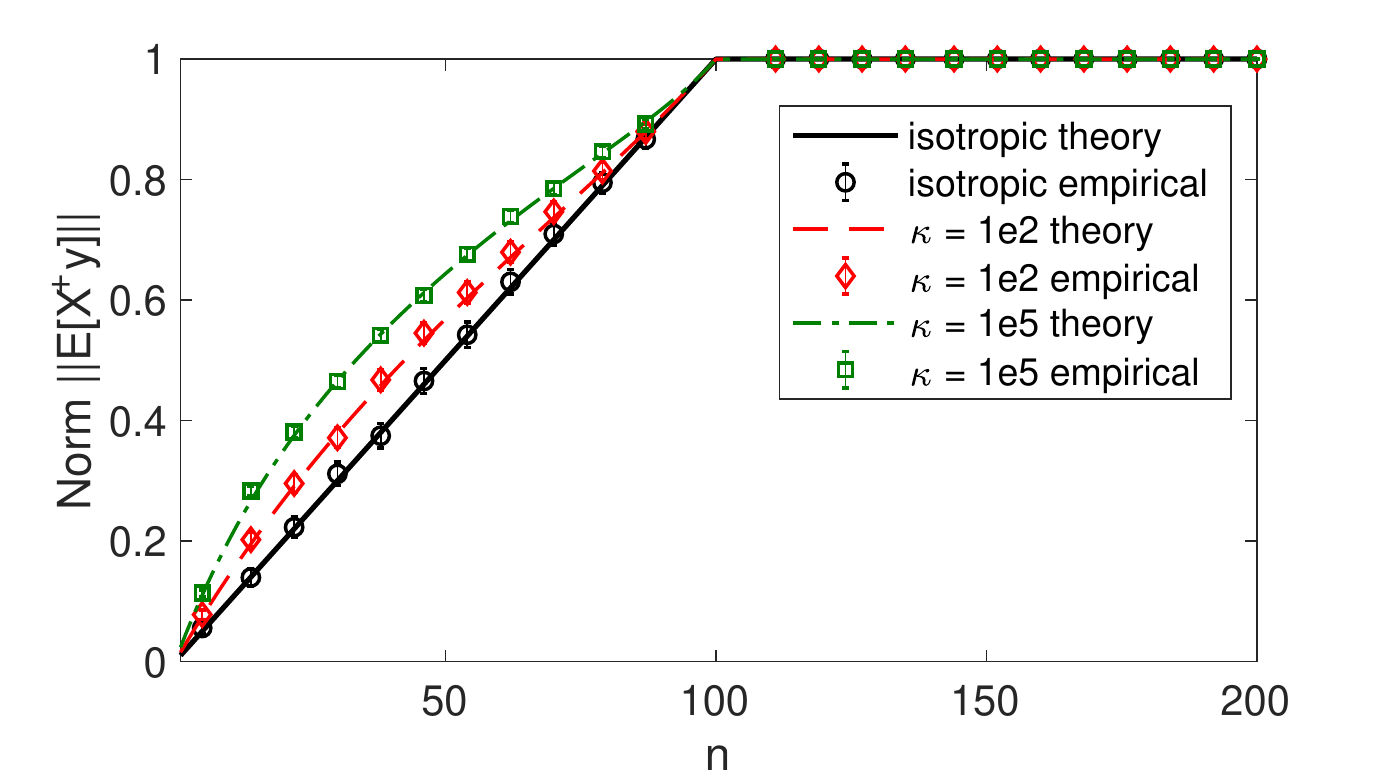}
    }
\caption{Illustration of the main results for $d=100$ and
$\mu=\Nc(\zero,\Sigmab)$ where $\Sigmab$ is diagonal with
eigenvalues decaying exponentially and scaled so that
$\tr(\Sigmab^{-1})=d$. We use our surrogate
formulas to plot (a) the MSE (Theorem \ref{t:mse}) and (b) the norm of the expectation (Theorem
\ref{t:unbiased}) of the Moore-Penrose estimator (\emph{theory}
lines), accompanied by the empirical estimates based on the standard
i.i.d.~design (error bars are three times the standard error of the
mean). We consider three different condition numbers $\kappa$ of
$\Sigmab$, with \emph{isotropic} corresponding to $\kappa=1$,
i.e., $\Sigmab=\I$. We use $\sigma^2=1$ and
$\w^*=\frac1{\sqrt{d}}\one$.}
\label{f:intro}
\end{figure}

\noindent
Proof of Theorem \ref{t:mse} is given in Appendix \ref{a:mse-proof}.
For illustration, we plot the MSE expressions in Figure~\ref{f:intro}a,
comparing them with empirical estimates of the true MSE under the
i.i.d.~design for a multivariate Gaussian distribution
$\mu=\Nc(\zero,\Sigmab)$ with several different covariance matrices $\Sigmab$. We keep the number of features $d$ fixed to
$100$ and vary the number of samples $n$, observing a double descent
peak at $n=d$. We observe that our theory aligns well with
the empirical estimates, whereas
previously, no such theory was available except for special
cases such as $\Sigmab=\I$ (more details in Theorem~\ref{t:asymptotic}
and Section~\ref{sec:asymp-conj-details}). The plots
show that varying the spectral decay of $\Sigmab$ has a significant effect on the
shape of the curve in the under-determined regime. We use the
horizontal line to denote the MSE of the null estimator
$\mathrm{MSE}[\zero]=\|\w^*\|^2=1$. When the eigenvalues of $\Sigmab$
decay rapidly, then the Moore-Penrose estimator suffers less error
than the null estimator for some values of $n<d$, and the curve
exhibits a local optimum in this regime.

One important aspect of Theorem~\ref{t:mse} comes from the relationship between $n$ and the parameter $\lambda_n$, which together satisfy $n=\tr(\Sigmab_\mu (\Sigmab_\mu+\lambda_n\I)^{-1})$.
This expression is precisely the classical notion of \emph{effective
  dimension} for ridge regression regularized with
$\lambda_n$~\citep{ridge-leverage-scores}, and it arises here even though there is
no explicit ridge regularization in the problem being considered in
Theorem~\ref{t:mse}. 
The global solution to the ridge regression task (i.e., $\ell_2$-regularized
least squares) with parameter $\lambda$ is defined as:
\begin{align*}
\argmin_\w \Big\{\E_{\mu,y}\big[\big(\x^\top\w-y(\x)\big)^2\big]
    + \lambda\|\w\|^2\Big\}\ =\ (\Sigmab_\mu +
  \lambda\I)^{-1}\v_{\mu,y},\quad\text{where } \v_{\mu,y}=\E_{\mu,y}[y(\x)\,\x ].
\end{align*}
When Assumption \ref{a:linear} holds, then
$\v_{\mu,y}=\Sigmab_\mu\w^*$, however ridge-regularized least squares
is well-defined for much more general response models.
Our second result makes a direct connection between the (expectation
of the) unregularized minimum norm solution on the sample
and the global ridge-regularized solution.
While the under-determined regime (i.e., $n<d$) is of primary interest to us,
for completeness we state this result for arbitrary values of $n$ and $d$.
Note that, just like the definition of regularized least squares, this
theorem applies more generally than Theorem~\ref{t:mse}, in that it
does \emph{not} require the responses to follow any linear model as in
Assumption~\ref{a:linear} (proof in Appendix~\ref{s:unbiased-proof}). 
\begin{theorem}[Implicit regularization of Moore-Penrose estimator]
\label{t:unbiased}
For $\mu$ satisfying
Assumption~\ref{a:general-position} and
  $y(\cdot)$ s.t.~$\v_{\mu,y}=\E_{\mu,y}[y(\x)\,\x]$ is well-defined,
  $\Xb\sim S_\mu^n$ (Definition \ref{d:surrogate}) and $\yb_i=y(\xbb_i)$,
  \begin{align*}
    \E\big[\Xb^\dagger\ybb\big] =
    \begin{cases}
       (\Sigmab_\mu + \lambda_n\I)^{-1}\v_{\mu,y} &\text{for }n<d,\\
        \Sigmab_\mu^{-1}\v_{\mu,y} &\text{for }n \ge d,
    \end{cases}
  \end{align*}
  where, as in Theorem \ref{t:mse}, $\lambda_n$ is such that the effective dimension
  $\tr(\Sigmab_\mu(\Sigmab_\mu+\lambda_n\I)^{-1})$ equals $n$.
\end{theorem}

\noindent
That is, when $n < d$, the Moore-Penrose estimator (which itself is
not regularized), computed on the
random training sample, in expectation equals the global ridge-regularized least
squares solution of the underlying regression
problem. Moreover, $\lambda_n$, i.e., the amount
of implicit $\ell_2$-regularization, is controlled by the degree of
over-parameterization in such a way as to ensure that $n$ becomes the ridge effective dimension
(a.k.a.~the effective degrees of freedom).

We illustrate this result in Figure
\ref{f:intro}b, plotting the norm of the expectation of the
Moore-Penrose estimator. As for the MSE, our surrogate theory aligns
well with the empirical estimates for i.i.d.~Gaussian designs, showing
that the shrinkage of the unregularized estimator in the
under-determined regime matches the implicit
ridge-regularization characterized by Theorem \ref{t:unbiased}. While the shrinkage
is a linear function of the sample size $n$ for isotropic features
(i.e., $\Sigmab=\I$), it
exhibits a non-linear behavior for other spectral decays.
Such \emph{implicit regularization} has been studied
previously~\citep[see, e.g.,][]{MO11-implementing, 
Mah12}; it has
been observed empirically for RandNLA sampling
algorithms~\citep{MMY15}; and it has also received attention more
generally within the context of neural networks~\citep{Ney17_TR}. While our implicit regularization result
is limited to the Moore-Penrose estimator, this new connection (and
others, described below) between the minimum norm solution of an unregularized
under-determined system and a ridge-regularized least squares solution
offers a simple interpretation for the implicit regularization
observed in modern machine learning architectures.

Our exact non-asymptotic expressions in Theorem~\ref{t:mse} and
our exact implicit regularization results in Theorem~\ref{t:unbiased}
are derived for the surrogate design, which is a
non-i.i.d.~distribution based on a determinantal point process. However, Figure \ref{f:intro}
suggests that those expressions accurately describe the MSE (up to lower order
terms) also under the standard i.i.d.~design $\X\sim\mu^n$ when $\mu$ is a multivariate Gaussian.
As a third result, we verify that the surrogate expressions for the
MSE are asymptotically consistent with the MSE of an i.i.d.~design,
for a wide class of distributions which include multivariate Gaussians.

\begin{theorem}[Asymptotic consistency of surrogate design]
  \label{t:asymptotic}
  Let $\X \in \mathbb{R}^{n \times d}$ have i.i.d.~rows $\x_i^\top = \z_i^\top\Sigmab^{\frac12}$
  where $\z_i$ has independent zero mean and unit variance
  sub-Gaussian entries, and suppose that Assumptions~\ref{a:linear}
  and \ref{a:general-position} are satisfied. Furthermore, suppose
  that there
  exist $c, C, C^* \in \mathbb{R}_{> 0}$ such that 
  $C \I \succeq \Sigmab \succeq c \I \succ 0$ and $\| \w^* \| \le
  C^*$. Then 
  \begin{align*}
   \MSE{\X^\dagger\y} - \Mc(\Sigmab, \w^*,\sigma^2,n) \to 0
  \end{align*}
  with probability one as $d,n \to \infty$ with $n/d \to \bar c \in (0,\infty) \setminus \{1\}$.
\end{theorem}

\noindent
The above result is particularly remarkable since our surrogate design
is a determinantal point process. DPPs are commonly used in ML to
ensure that the data points in a sample are well spread-out. However,
if the data distribution is sufficiently regular (e.g., a multivariate
Gaussian), then the i.i.d.~samples are already spread-out reasonably
well, so rescaling the distribution by a determinant has a negligible
effect that vanishes in the high-dimensional regime.
Furthermore, our empirical estimates (Figure~\ref{f:intro}) suggest that the surrogate expressions
are accurate not only in the asymptotic limit, but even for moderately large
dimensions. Based on a detailed empirical analysis described in
Section \ref{sec:asymp-conj-details}, 
we conjecture that the convergence described in
Theorem~\ref{t:asymptotic} has the rate of~$O(1/d)$.

\section{Related work}
\label{s:related-work}

There is a large body of related work, which for simplicity we cluster into three groups.

\textbf{Double descent.}
The double descent phenomenon has been observed empirically
in a number of learning models, including neural networks
\citep{BHMM19,GJSx19_TR}, kernel methods \citep{BMM18_TR,BRT18_TR},
nearest neighbor models \citep{BHM18_TR}, and decision trees \citep{BHMM19}. The
theoretical analysis of double descent, and more broadly the generalization
properties of interpolating estimators, have primarily focused on various forms of
linear regression \citep{BLLT19_TR,LR18_TR,HMRT19_TR,MVSS19_TR}. Note
that while we analyze the classical mean squared error, many 
works focus on the squared prediction error.
Also, unlike in our work, some of the literature on double descent deals with linear regression in the
so-called \emph{misspecified} setting, where the set of observed
features does not match the feature space in
which the response model is linear
\citep{belkin2019two,HMRT19_TR,Mit19_TR,MM19_TR}, e.g., when the
learner observes a random subset of $d$ features from a larger
population.

\begin{wrapfigure}{r}{\ifisarxiv 0.5\else 0.44\fi\textwidth}
 \ifisarxiv\else \vspace{-.2cm}\fi
  \centering
 \includegraphics[width=.43\textwidth]{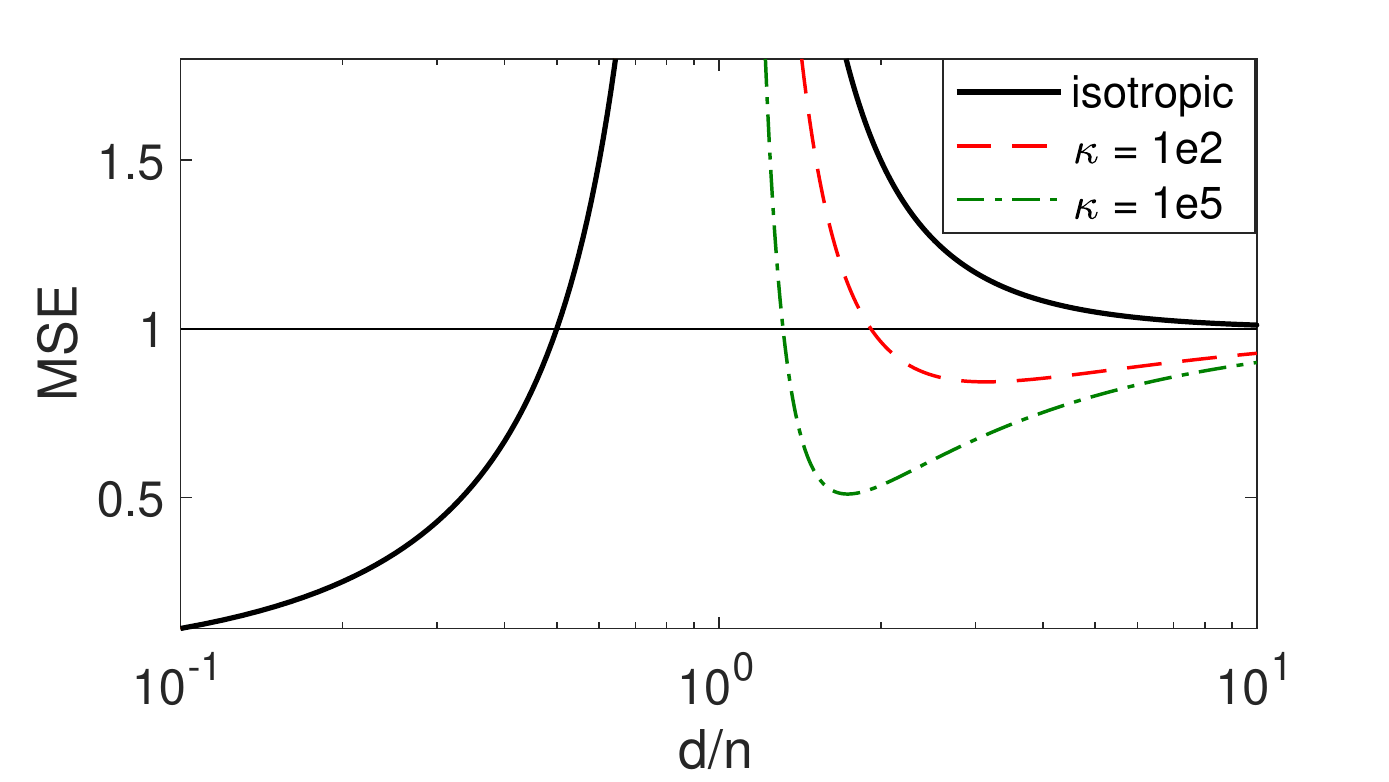}
\ifisarxiv\else \vspace{-3mm}\fi
  \caption{
Surrogate MSE as a function of $d/n$, with $n$
fixed to $100$ and varying $d$, for signal-to-noise ratio $\mathrm{SNR}=1$.}
  \label{f:model}
\ifisarxiv\else  \vspace{-.5cm}\fi
\end{wrapfigure}

The most directly comparable to our setting is the recent work of
\cite{HMRT19_TR}. They study how
varying the feature dimension affects the (asymptotic) generalization
error for linear regression, 
however their analysis is limited to certain special settings such as an
isotropic data distribution. As an additional
point of comparison, in Figure~\ref{f:model} we plot the MSE
expressions of Theorem~\ref{t:mse} when varying the
feature dimension~$d$ (the setup is the same as in
Figure~\ref{f:intro}). Our plots follow the  
trends outlined by \cite{HMRT19_TR} for the isotropic case (see their Figure~2),
but the spectral decay of the covariance (captured by our new
MSE expressions) has a significant effect on the descent curve.
This leads to generalization in the under-determined regime even when
the signal-to-noise ratio ($\mathrm{SNR}=\|\w^*\|^2/\sigma^2$) is 1,
unlike suggested by \cite{HMRT19_TR}.

\textbf{RandNLA and DPPs.}
Randomized Numerical Linear Algebra
\citep{DM16_CACM,RandNLA_PCMIchapter_TR} has
traditionally focused on obtaining purely algorithmic improvements for tasks
such as least squares regression, 
but there has been growing
interest in understanding the statistical properties of these
randomized methods \citep{MMY15,GarveshMahoney_JMLR}.
Determinantal point processes \citep{dpp-ml} have been recently shown to combine
strong worst-case regression guarantees with elegant statistical
properties \citep{unbiased-estimates}.
However, these results are limited to the over-determined setting \citep{leveraged-volume-sampling,correcting-bias,minimax-experimental-design} and
ridge regression
\citep{regularized-volume-sampling,bayesian-experimental-design}.
Our results are also related to recent work on using DPPs 
to analyze the expectation of the 
inverse~\citep{determinantal-averaging} and generalized
inverse~\citep{MDK19_TR} of a subsampled matrix.

\textbf{Implicit regularization.}
The term implicit regularization typically refers to the notion that approximate
computation 
can implicitly lead to
statistical regularization. 
See~\cite{MO11-implementing, PM11, GM14_ICML} and references therein
for early work on the topic; and see~\cite{Mah12} for an overview.
More recently, often motivated by neural networks, there has been work
on implicit regularization that typically considered SGD-based
optimization algorithms.
See, e.g., theoretical results~\citep{NTS14_TR,Ney17_TR,SHNx17_TR,GWBNx17,ACHL19,KBMM19_TR}
as well as extensive empirical studies~\citep{MM18_TR,MM19_HTSR_ICML}.
The implicit regularization observed by us is different in that it is
not caused by an inexact approximation algorithm (such as SGD) but rather by the
selection of one out of many exact solutions (e.g., the minimum norm
solution). In this context, most relevant are the
asymptotic results of \cite{KLS18_TR} and \cite{LJB19_TR}.

\section{Surrogate random designs}
\label{s:determinantal}

In this section, we provide the definition of our surrogate random
design $S_\mu^n$, where $\mu$ is a $d$-variate probability measure and
$n$ is the sample size. This distribution is used in place
of the standard random design $\mu^n$ consisting of $n$ row vectors drawn
independently from $\mu$.

\textbf{Preliminaries.}
For an $n\times n$ matrix $\A$, we use $\pdet(\A)$ to denote the pseudo-determinant of $\A$,
which is the product of non-zero eigenvalues.
For index subsets $\Ic$
and $\Jc$, we use $\A_{\Ic,\Jc}$ to denote the submatrix of $\A$ with
rows indexed by $\Ic$ and columns indexed by $\Jc$. We may write
$\A_{\Ic,*}$ to indicate that we take a subset of rows.
We let $\X\sim\mu^k$ denote a $k\times d$ random matrix with rows
drawn i.i.d.~according to $\mu$, and the $i$th row is denoted as $\x_i^\top$.
We also let $\Sigmab_\mu=\E_{\mu}[\x\x^\top]$, where $\E_{\mu}$ refers to
the expectation with respect to $\x^\top\!\sim\mu$, assuming throughout that
$\Sigmab_\mu$ is well-defined and positive definite.
We use $\Poisson(\gamma)_{\leq a}$ as the Poisson distribution
restricted to $[0,a]$, whereas $\Poisson(\gamma)_{\geq a}$ is restricted
to $[a,\infty)$.
We also let $\#(\X)$ denote the number of rows of $\X$.

\begin{definition}\label{d:det}
  Let $\mu$ satisfy Assumption~\ref{a:general-position} and let $K$ be
  a random variable over $\mathbb{Z}_{\geq 0}$. A determinantal design
    $\Xb\sim \Det(\mu,K)$ is a
distribution with the same domain as $\X\sim\mu^K$ such that~for any
event $E$ measurable w.r.t.~$\X$, we have\vspace{-2mm}
\begin{align*}
\Pr\big\{\Xb\in E\big\}\ = \frac{\E[\pdet(\X\X^\top)\one_{[\X\in E]}]}{\E[\pdet(\X\X^\top)]}.
\end{align*}
\end{definition}\vspace{-2mm}

\noindent
The above definition can be
interpreted as rescaling the density function of $\mu^K$ by the
pseudo-determinant, and then renormalizing it. 
We now construct our surrogate design $S_\mu^n$ by appropriately
selecting the random variable $K$.
The obvious choice of $K=n$ does \emph{not} result in simple closed
form expressions for the MSE in the under-determined regime (i.e.,
$n<d$), which is the regime of primary interest to us. 
Instead, we derive our random variables $K$ from the Poisson
distribution. 
\begin{definition}\label{d:surrogate}
For $\mu$ satisfying Assumption~\ref{a:general-position},
define surrogate design $S_\mu^n$ as $\Det(\mu,K)$ where:
\vspace{-2mm}
    \begin{enumerate}
\item if $n<d$, then $K\sim \Poisson(\gamma_n)_{\leq d}$ with
 $\gamma_n$ as the solution of
$n=\tr(\Sigmab_\mu(\Sigmab_\mu+\frac1{\gamma_n}\I)^{-1})$,
 \vspace{-2mm}
\item if $n=d$, then we simply let $K=d$,
  \vspace{-2mm}
\item if $n>d$, then $K\sim\Poisson(\gamma_n)_{\geq d}$ with $\gamma_n=n-d$.
\end{enumerate}
\end{definition}

\noindent
Note that the under-determined case, i.e., $n<d$, is restricted to $K\leq d$ so that, under Assumption~\ref{a:general-position}, $\pdet(\X\X^\top)=\det(\X\X^\top)$ with probability 1.
On the other hand in the over-determined case, i.e., $n>d$, we have
$K\geq d$ so that $\pdet(\X\X^\top)=\det(\X^\top\X)$. In the special case
of $n=d=K$ both of these equations are satisfied: $\pdet(\X\X^\top)=\det(\X^\top\X)=\det(\X\X^\top)=\det(\X)^2$.

The first non-trivial property of the surrogate design $S_\mu^n$ is
that the expected sample size is in fact always equal to $n$, which we
prove in Appendix \ref{appx: proof-of-l-size}.
\begin{lemma} \label{l:size}
Let $\Xb\sim S_\mu^n$ for any $n>0$.
 Then, we have $\E[\#(\Xb)] = n$.
\end{lemma}

\noindent
 Our general template for computing expectations under
 a surrogate design $\Xb\sim\S_\mu^n$ is to use the following expressions based on the
i.i.d.~random design $\X\sim\mu^K$:
\begin{align}
  \E[F(\Xb)] &=\begin{cases}
    \frac{\E[\det(\X\X^\top)F(\X)]}{\E[\det(\X\X^\top)]}
   \quad K\sim\Poisson(\gamma_n)&\text{for }n<d,\\[2mm]
    \frac{\E[\det(\X)^2F(\X)]}{\E[\det(\X)^2]}\hspace{3mm}
    \quad K=d&\text{for }n=d,\\[2mm]
    \frac{\E[\det(\X^\top\X)F(\X)]}{\E[\det(\X^\top\X)]}
   \quad K\sim\Poisson(\gamma_n)&\text{for }n>d.
  \end{cases}\label{eq:cases}
\end{align}
These formulas follow from Definitions \ref{d:det} and
\ref{d:surrogate} because the determinants $\det(\X\X^\top)$ and
$\det(\X^\top\X)$ are non-zero precisely in the regimes $n\leq d$ and
$n\geq d$, respectively, which is why we can drop the restrictions on the
range of the Poisson distribution.
We compute the normalization constants
by introducing the concept of determinant preserving random matrices,
discussed in Section \ref{s:dp}.

\paragraph{Proof sketch of Theorem \ref{t:mse}} 
We focus here on the under-determined regime (i.e., $n<d$),
highlighting the key new expectation formulas we develop to derive the
MSE expressions for surrogate designs. A standard decomposition of the MSE yields:
\begin{align}
  \MSE{\Xb^\dagger\ybb}
  &= \E\big[\|\Xb^\dagger(\Xb\w^*+\xib)-\w^*\|^2\big]
  =\sigma^2\E\big[\tr\big((\Xb^\top\Xb)^{\dagger}\big)\big] +
    \w^{*\top}\E\big[\I-\Xb^\dagger\Xb\big]\w^*.\label{eq:mse-derivation}
\end{align}
Thus, our task is to find closed form expressions for the two
expectations above. The latter, which is the expected projection onto
the complement of the row-span of $\Xb$, is proven in
Appendix \ref{s:unbiased-proof}.
\begin{lemma}\label{l:proj}
If  $\Xb\sim S_\mu^n$ and $n<d$, then we have:
$\E\big[\I-\Xb^\dagger\Xb\big] = (\gamma_n\Sigmab_\mu+\I)^{-1}$.
\end{lemma}

\noindent
No such expectation formula is known for i.i.d.~designs, except when
$\mu$ is an isotropic Gaussian. In Appendix \ref{s:unbiased-proof}, we
also prove a generalization of Lemma \ref{l:proj} which is then used
to establish our implicit regularization result
(Theorem~\ref{t:unbiased}). We next give an expectation formula 
for the trace of the Moore-Penrose inverse of the covariance
matrix for a surrogate design (proof in Appendix \ref{a:mse-proof}).
\begin{lemma}\label{l:sqinv-all}
If  $\Xb\sim S_\mu^n$ and $n<d$, then:
$\E\big[\tr\big((\Xb^\top\Xb)^{\dagger}\big)\big]
=\gamma_n\big(1-
\det\!\big((\tfrac1{\gamma_n}\I+\Sigmab_\mu)^{-1}\Sigmab_\mu\big)\big)$.
\end{lemma}

\noindent
Note the implicit regularization term which appears in both formulas, 
given by $\lambda_n=\frac1{\gamma_n}$. Since $n =
\tr(\Sigmab_\mu(\Sigmab_\mu+\lambda_n\I)^{-1})=d-\lambda_n\tr((\Sigmab_\mu+\lambda_n\I)^{-1})$,
it follows that $\lambda_n=(d-n)/\tr((\Sigmab_\mu+\lambda_n\I)^{-1})$.
Combining this with Lemmas \ref{l:proj} and \ref{l:sqinv-all}, we
recover the surrogate MSE expression in Theorem~\ref{t:mse}.

\section{Determinant preserving random matrices}
\label{s:dp}

In this section, we introduce the key tool for computing expectation formulas of matrix determinants.
It is used in our analysis of the surrogate design, and it should be of independent interest.

The key question motivating the following definition is: \textit{When does taking expectation commute with computing a determinant for a square random matrix?}
\begin{definition}\label{d:main}
A random $d\times d$ matrix $\A$ is called determinant
  preserving (d.p.), if
\begin{align*}
  \E\big[\!\det(\A_{\Ic,\Jc})\big] =
  \det\!\big(\E[\A_{\Ic,\Jc}]\big)\quad \text{for all }\Ic,\Jc\subseteq
  [d]\text{ s.t. }|\Ic|=|\Jc|.
\end{align*}
\end{definition}

\noindent
We next give a few simple examples to provide some intuition. First, note
that every $1\times 1$ random matrix is determinant preserving simply
because taking a determinant is an identity transfomation in one
dimension. Similarly, every fixed matrix is determinant preserving because
in this case taking the expectation is an identity
transformation. In all other cases, however, Definition \ref{d:main}
has to be verified more carefully. Further examples (positive and
negative) follow.
\begin{example}
If $\A$ has i.i.d. Gaussian entries $a_{ij}\sim\Nc(0,1)$, then
$\A$ is d.p.~because $\E[\det(\A)]=0$.
\end{example}

\noindent
In fact, it can be shown that all random matrices with independent entries
are determinant preserving. However, this is not a necessary condition.
\begin{example}\label{e:rank-1}
Let $\A = s\,\Z$, where $\Z$ is fixed with $\rank(\Z) = r$, and $s$
is a scalar random variable. Then for $|\Ic|=|\Jc|=r$ we have
\begin{align*}
  \E\big[\det(s\,\Z_{\Ic,\Jc})\big] &= \E[s^r]\det(\Z_{\Ic,\Jc})
                                      =\det\Big(\big(\E[s^r]\big)^{\frac1r}\,\Z_{\Ic,\Jc}\Big),
\end{align*}
  so if $r=1$ then $\A$ is determinant preserving, whereas if $r>1$
  and $\Var[s]>0$ then it is not.
\end{example}

\noindent
To construct more complex examples, we show that determinant preserving random matrices are
closed under addition and multiplication. The proof of this result is
an extension of an existing argument, given by
\cite{determinantal-averaging} in the proof of Lemma~7, for computing
the expected determinant of the sum of rank-1 random matrices (proof in Appendix \ref{a:dp}). 
\begin{lemma}[Closure properties]\label{t:ring}
  If $\A$ and $\B$ are independent and determinant preserving, then:
\vspace{-1mm}
  \begin{enumerate}
  \item $\A+\B$ is determinant preserving,
    \vspace{-2mm}
  \item $\A\B$ is determinant preserving.
    \end{enumerate}
\end{lemma}

\noindent
Next, we introduce another important class of d.p.~matrices:
a sum of i.i.d.~rank-1 random matrices with the number of
i.i.d.~samples being a Poisson random variable. Our use of the Poisson
distribution is crucial for the below result to hold. It is an
extension of an expectation formula given by \cite{dpp-intermediate}
for sampling from discrete distributions (proof in Appendix \ref{a:dp}).
\begin{lemma}\label{l:poisson}
If $K$ is a Poisson random variable and $\A,\B$ are random $K\times d$
matrices whose rows  are sampled as an i.i.d.~sequence of joint pairs of
random vectors, then $\A^\top\B$ is d.p., and so:
\begin{align*}
  \E\big[\det(\A^\top\B)\big] &= \det\!\big(\E[\A^\top\B]\big).
  \end{align*}
\end{lemma}

\noindent
Finally, we show the expectation formula needed for obtaining the
normalization constant of the under-determined surrogate design, given
in \eqref{eq:cases}.
The below result is more general than the normalization constant
requires, because it allows the matrices $\A$ and $\B$ to be different
(the constant is obtained by setting $\A=\B=\X\sim\mu^K$).
In fact, we use this more general statement to show Theorems
\ref{t:mse} and~\ref{t:unbiased}. The proof uses
Lemmas \ref{t:ring} and \ref{l:poisson} (see Appendix \ref{a:dp}).
\begin{lemma}\label{l:normalization}
If $K$ is a Poisson random variable and $\A$, $\B$ are random $K\times d$
matrices whose rows  are sampled as an i.i.d.~sequence of joint pairs of
random vectors, then
\begin{align*}
  \E\big[\det(\A\B^\top)\big] &= \ee^{-\E[K]}\det\!\big(\I + \E[\B^\top\A]\big).
  \end{align*}
\end{lemma}

\begin{figure}[t]
  \includegraphics[width=\textwidth]{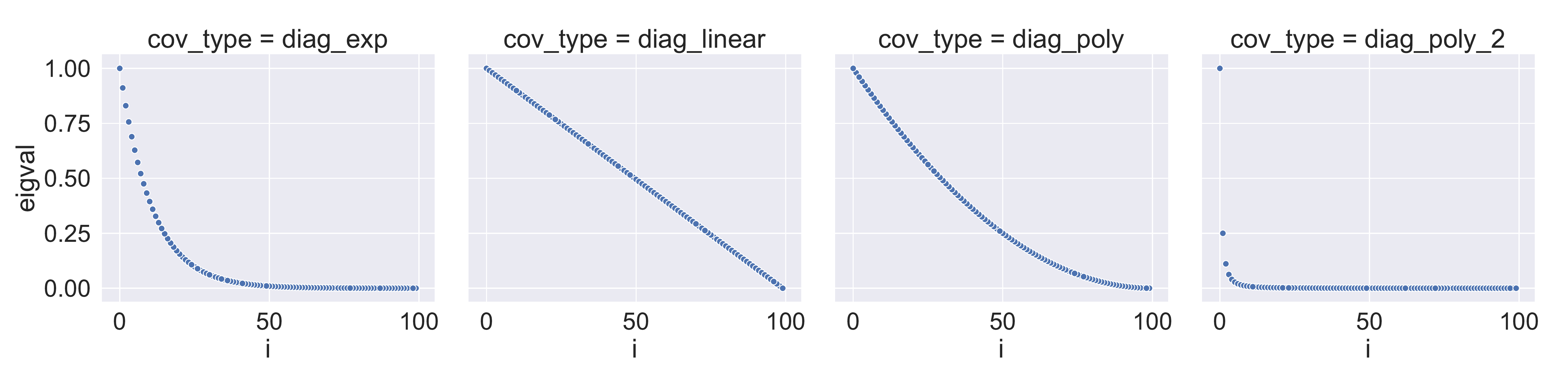}
  \vspace{-1cm}
  \caption{Scree-plots of $\Sigmab$ for the eigenvalue decays examined
    in our empirical valuations.  }
  \label{fig:eig-decays}
\end{figure}

\section{Empirical evaluation of asymptotic consistency}
\label{sec:asymp-conj-details}

In this section, we empirically quantify the convergence rates for
the asymptotic result of Theorem~\ref{t:asymptotic}.
We focus on the under-determined regime (i.e.,
$n<d$) and separate the evaluation into the bias and
variance terms, following the MSE decomposition given
in \eqref{eq:mse-derivation}. Consider  $\X = \Z\Sigmab^{1/2} $, where the entries of $\Z$ are
i.i.d. standard Gaussian, and define:\vspace{-1mm}
\begin{enumerate}
  \item Variance discrepancy:\quad
    $\big|\frac{\E[\tr((\X^\top\X)^\dagger)]}{\Vc(\Sigmab,n)}-1\big|$ where
    $\Vc(\Sigmab,n)=\frac{1-\alpha_n}{\lambda_n}$.
  \item Bias discrepancy:\quad
     $\sup_{\w\in\R^d\backslash\{\zero\}}\big|\frac{\w^\top\E[\I-\X^\dagger\X]\w}
     {\w^\top\Bc(\Sigmab,n)\w} - 1\big|$
    where $\Bc(\Sigmab,n) = \lambda_n(\Sigmab+\lambda_n\I)^{-1}$.
  \end{enumerate}\vspace{-1mm}
   Recall that $\lambda_n=\frac {d-n}{\tr((\Sigmab+\lambda_n\I)^{-1})}$,
so our surrogate MSE can be written as
$\Mc=\sigma^2\Vc(\Sigmab,n)+\w^{*\top}\Bc(\Sigmab,n)\w^*$, and when both
discrepancies are bounded by $\epsilon$, then $(1-2\epsilon)\Mc\leq\MSE{\X^\dagger\y}\leq (1+2\epsilon)\Mc$.
In our experiments, we consider four standard eigenvalue decay profiles
for $\Sigmab$, including polynomial and exponential decay (see
 \Cref{fig:eig-decays} and \Cref{sec:eig-decay-details}).
\begin{figure} 
  \includegraphics[width=\textwidth]{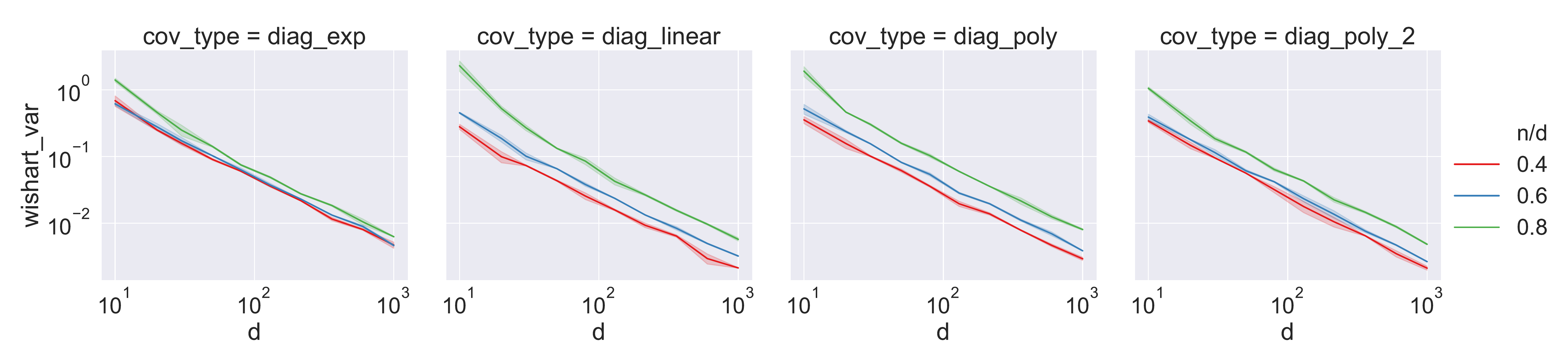}
  \includegraphics[width=\textwidth]{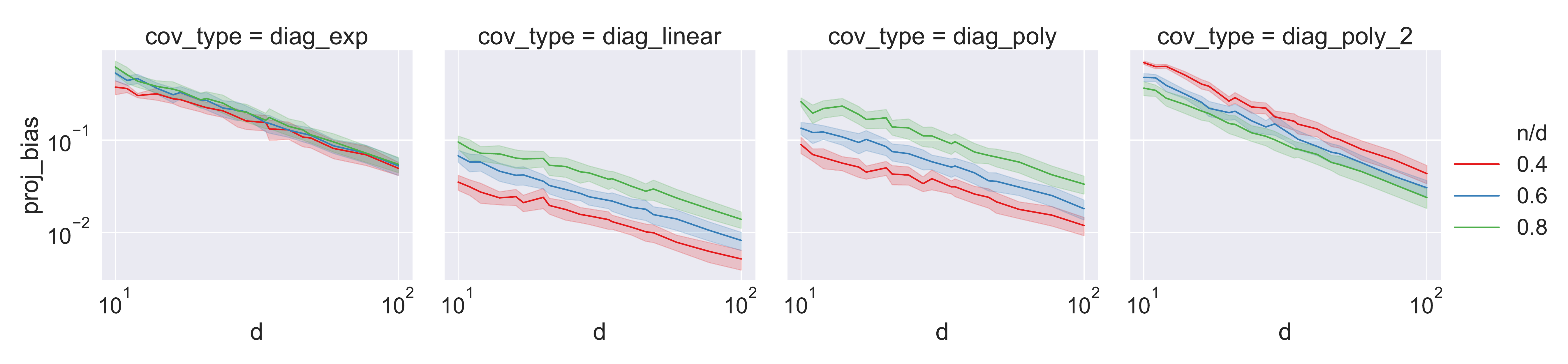}
  \vspace{-.8cm}
  \caption{
    Empirical verification of the asymptotic consistency of surrogate MSE.
    We show the discrepancies for the variance (top) and bias
    (bottom),  with bootstrapped $95\%$ confidence intervals, as $d$
    increases and $n/d$ is fixed. We observe
     $O(1/d)$ decay (linear with slope $-1$ on a log-log plot).
  }
  \label{f:conj-wishart}
\end{figure}

Figure~\ref{f:conj-wishart} (top) plots the variance discrepancy (with
$\E[\tr((\X^\top\X)^\dagger)]$ estimated via Monte Carlo
sampling and bootstrapped confidence intervals) as $d$ increases from $10$ to
$1000$, across a range of aspect ratios $n/d$. In all cases, we observe that
the discrepancy decays to zero at a rate of $O(1/d)$.
Figure~\ref{f:conj-wishart} (bottom) plots the bias discrepancy, with the same
rate of decay observed throughout.  Note that the range of $d$ is smaller than
in Figure \ref{f:conj-wishart} (top) because the large number of Monte Carlo
samples (up to two million) required for this experiment made the computations
much more expensive (more details in Appendix \ref{a:empirical}). Based on the
above empirical results, we conclude with a conjecture.
\begin{conjecture}
  \label{c:1-over-d-rate}
  When $\mu$ is a centered multivariate Gaussian and its covariance
  has a constant condition
  number, then, for $n/d$ fixed, the surrogate MSE satisfies:
  $\big|\frac{\textnormal{MSE}[\X^\dagger\y]}{\Mc}-1\big|= O(1/d)$.
\end{conjecture}

\section{Conclusions}
\label{s:conclusions}

We derived exact non-asymptotic expressions for the MSE of the
Moore-Penrose estimator in the linear regression task, reproducing
the double descent phenomenon as the sample size crosses between the
under- and over-determined regime. To achieve this, we modified the
standard i.i.d.~random design distribution using a determinantal
point process to obtain a surrogate design which admits exact MSE expressions,
while capturing the key properties of the i.i.d.~design. We
also provided a result that relates the expected value of the
Moore-Penrose estimator of a training sample in the under-determined regime (i.e., the
minimum norm solution) to the ridge-regularized least squares solution
for the population distribution, thereby providing an interpretation for the
implicit regularization resulting from over-parameterization.

\paragraph{Acknowledgements.}
We would like to acknowledge ARO, DARPA, NSF, ONR, and GFSD for providing
partial support of this work. We also thank Zhenyu Liao for pointing
out fruitful connections between our results and the asymptotic
analysis of random matrix resolvents.

\ifisarxiv
\bibliographystyle{plainnat}
\else
\bibliographystyle{alpha}
\fi

\bibliography{../pap}

\ifisarxiv\else\newpage\fi

\appendix

\section{Proof of Lemma~\ref{l:size}}
\label{appx: proof-of-l-size}

We first record an important property of the design $S_\mu^d$
which can be used to construct an over-determined design for any $n>d$. A similar
version of this result was also previously shown by
\cite{correcting-bias-journal} for a different determinantal design.

\begin{lemma}\label{l:decomposition}
  Let $\Xb\sim S_\mu^d$ and $\X\sim \mu^K$, where
  $K\sim\Poisson(\gamma)$. Then the matrix composed of a random
  permutation of the rows from $\Xb$ and $\X$ is distributed according to
  $S_\mu^{d+\gamma}$.
\end{lemma}

\begin{proof}
Let $\Xt$ denote the matrix constructed from the permuted rows of
$\Xb$ and $\X$.  Letting $\Z\sim\mu^{K+d}$, we derive the probability
$\Pr\big\{\Xt\!\in\! E\big\}$ by summing over the possible index subsets  $S\subseteq
[K+d]$ that correspond to the rows coming from $\Xb$:
\begin{align*}
  \Pr\big\{\Xt\in E\big\} &= \E\bigg[\frac{1}{\binom{K+d}{d}}
  \sum_{S:\,|S|=d}\frac{\E[\det(\Z_{S,*})^2\one_{[\Z\in E]}\mid
  K]}{d!\det(\Sigmab_\mu)}\bigg]\\
  &=\sum_{k=0}^\infty
    \frac{\gamma^k\ee^{-\gamma}}{k!}\,\frac{\gamma^dk!}{(k+d)!}\,
    \frac{\E\big[\sum_{S:\,|S|=d}\det(\Z_{S,*})^2\one_{[\Z\in E]}\mid
    K=k\big]}{\det(\gamma\Sigmab_\mu)}\\
  &\overset{(*)}{=} \sum_{k=0}^\infty
    \frac{\gamma^{k+d}\ee^{-\gamma}}{(k+d)!}
    \,\frac{\E[\det(\Z^\top\Z)\one_{[\Z\in E]}\mid K=k]}{\det(\gamma\Sigmab_\mu)},
\end{align*}
where $(*)$ uses the Cauchy-Binet formula to sum over all subsets $S$
of size $d$. Finally, since the sum shifts from $k$
to $k+d$, the last expression can be rewritten as
$\E[\det(\X^\top\X)\one_{[\X\in E]}]/\det(\gamma\Sigmab_\mu)$, where recall that
$\X\sim\mu^K$ and $K\sim\Poisson(\gamma)$, matching the definition of $S_\mu^{d+\gamma}$.
\end{proof}

We now proceed with the proof of Lemma \ref{l:size}, where we establish
that the expected sample size of $S_\mu^n$ is indeed $n$.

\begin{proofof}{Lemma}{\ref{l:size}}
  The result is obvious when $n=d$, whereas
  for $n>d$ it is an immediate consequence
  of Lemma \ref{l:decomposition}.
  Finally, for $n<d$ the expected sample
  size follows as a corollary of Lemma \ref{l:proj}, which states that
  \begin{align*}
\text{(Lemma \ref{l:proj})} \qquad\E\big[\I - \Xb^\dagger\Xb\big] =
    (\gamma_n\Sigmab_\mu + \I)^{-1},
  \end{align*}
  where $\Xb^\dagger\Xb$ is the orthogonal projection onto
  the subspace spanned by the rows of $\Xb$. Since the rank of this
  subspace is equal to the number of the rows, we have
  $\#(\Xb)=\tr(\Xb^\dagger\Xb)$, so
  \begin{align*}
    \E\big[\#(\Xb)\big] = d - \tr\big((\gamma_n\Sigmab_\mu +
    \I)^{-1}\big) =
    \tr\big(\gamma_n\Sigmab_\mu(\gamma_n\Sigmab_\mu+\I)^{-1}\big) = n,
  \end{align*}
  which completes the proof.
\end{proofof}

\section{Proofs for Section \ref{s:dp}}
\label{a:dp}

We use $\adj(\A)$ to denote the adjugate of $\A$, defined as follows: the
$(i,j)$th entry of $\adj(\A)$ is
$(-1)^{i+j}\det(\A_{[n]\backslash\{j\},[n]\backslash\{i\}})$.
We will use two useful identities related to the adjugate: (1)
$\adj(\A)=\det(\A)\A^{-1}$ for invertible $\A$, and (2)
$\det(\A+\u\v^\top)=\det(\A)+\v^\top\!\adj(\A)\u$
\citep[see Fact 2.14.2 in][]{matrix-mathematics}.

First, note that from the definition of an adjugate matrix it immediately follows that if $\A$ is
determinant preserving then adjugate commutes with expectation for this matrix:
\begin{align}
  \E\big[\big(\!\adj(\A)\big)_{i,j}\big] &=
  \E\big[(-1)^{i+j}\det(\A_{[d]\backslash\{j\},[d]\backslash\{i\}})\big]\nonumber
  \\
&=(-1)^{i+j}\det\!\big(\E[\A_{[d]\backslash\{j\},[d]\backslash\{i\}}]\big)
  \\
  &= \big(\!\adj(\E[\A])\big)_{i,j}.\label{eq:adj}
\end{align}

\begin{proofof}{Lemma}{\ref{t:ring}} \
 First, we show that $\A+\u\v^\top$ is d.p.~for any fixed
 $\u,\v\in\R^d$. Below, we use the identity for a rank one
 update of a determinant:
 $\det(\A+\u\v^\top)=\det(\A)+\v^\top\!\adj(\A)\u$. It follows that
 for any $\Ic$ and $\Jc$ of the same size,
  \begin{align*}
\E\big[\!\det(\A_{\Ic,\Jc}\!+\u_{\Ic}\v_{\Jc}^\top)\big] &=
    \E\big[\!\det(\A_{\Ic,\Jc}) +
    \v_{\Jc}^\top\adj(\A_{\Ic,\Jc}) \u_{\Ic}\big]\\
    &\overset{(*)}{=}\det\!\big(\E[\A_{\Ic,\Jc}]\big) +
      \v_{\Jc}^\top\adj\!\big(\E[\A_{\Ic,\Jc}]\big) \u_{\Ic}\\
    &=\det\!\big(\E[\A_{\Ic,\Jc} \!+ \u_{\Ic}\v_{\Jc}^\top]\big),
  \end{align*}
  where $(*)$ used \eqref{eq:adj}, i.e., the fact that for d.p.~matrices, adjugate commutes
  with expectation. Crucially, through the definition of an adjugate
  this step implicitly relies on the assumption that all the square
  submatrices of $\A_{\Ic,\Jc}$ are also  determinant preserving.
  Iterating this, we get that $\A+\Z$ is d.p.~for any fixed
  $\Z$. We now show the same for $\A+\B$:
  \begin{align*}
\E\big[\!\det(\A_{\Ic,\Jc}\!+\B_{\Ic,\Jc})\big]
    &=
      \E\Big[\E\big[\!\det(\A_{\Ic,\Jc}\!+\B_{\Ic,\Jc})\mid\B\big]\Big]\\
    &\overset{(*)}{=}\E\Big[\!\det\!\big(\E[\A_{\Ic,\Jc}]\!+\B_{\Ic,\Jc}\big)\Big]\\
      &= \det\!\big(\E[\A_{\Ic,\Jc}\!+\B_{\Ic,\Jc}]\big),
  \end{align*}
  where $(*)$  uses the fact that after conditioning on $\B$ we can
  treat it as a fixed matrix. Next, we show that $\A\B$ is determinant preserving via the Cauchy-Binet formula:
  \begin{align*}
    \E\big[\!\det\!\big((\A\B)_{\Ic,\Jc}\big)\big]
    &= \E\big[\!\det(\A_{\Ic,*}\B_{*,\Jc})\big]\\
    &=\E\bigg[\sum_{S:\,|S|=|\Ic|}\!\!\det\!\big(\A_{\Ic,S}\big)
      \det\!\big(\B_{S,\Jc}\big)\bigg]\\
&=\!\!\sum_{S:\,|S|=|\Ic|}\!\!\det\!\big(\E[\A]_{\Ic,S}\big)
                                                \det\!\big(\E[\B]_{S,\Jc}\big)\\
    &=\det\!\big(\E[\A]_{\Ic,*}\, \E[\B]_{*,\Jc}\big)\\
      &= \det\!\big(\E[\A\B]_{\Ic,\Jc}\big),
  \end{align*}
  where recall that $\A_{\Ic,*}$ denotes the submatrix of $\A$
  consisting of its (entire) rows indexed by $\Ic$.
  \end{proofof}

To prove Lemma \ref{l:poisson}, we will use the following
lemma, many variants of which appeared in the literature
\cite[e.g.,][]{expected-generalized-variance}. We use the one given by
\cite{correcting-bias}.
\begin{lemma}[\cite{correcting-bias}]\label{l:cb}
If the rows of random $k\times d$ matrices $\A,\B$
  are sampled as an i.i.d.~sequence of $k\geq d$ pairs of joint random vectors, then
\begin{align}
  k^d\,\E \big[\det(\A^\top\B)\big]
  &= \ktd\,\det\!\big(\E[\A^\top\B]\big).
     \end{align}
 \end{lemma}

\noindent
Here, we use the following standard shorthand: $\ktd =
\frac{k!}{(k-d)!} = k\,(k-1)\dotsm(k-d+1)$. Note that the above result
almost looks like we are claiming that the matrix $\A^\top\B$ is d.p.,
but in fact it is not because $k^d\neq \ktd$. The difference
in those factors is precisely what we are going to correct with the
Poisson random variable. We now present the proof of Lemma
\ref{l:poisson}.
\begin{proofof}{Lemma}{\ref{l:poisson}}
Without loss of generality, it suffices to check Definition \ref{d:main} with both $\Ic$ and
$\Jc$ equal $[d]$. We first expand the expectation by
conditioning on the value of $K$ and letting $\gamma=\E[K]$:
    \begin{align*}
      \E\big[\!\det(\A^\top\B)\big]
      &= \sum_{k=0}^\infty
\E\big[\det(\A^\top\B)\mid K\!=\!k\big]\
\Pr(K\!=\!k)\\
      \text{(Lemma \ref{l:cb})}
      \quad&=
        \sum_{k=d}^\infty\frac{k! k^{-d}}{(k-d)!}\det\!\big(\E[\A^\top\B\mid
        K\!=\!k]\big)
        \frac{\gamma^k\ee^{-\gamma}}{k!}\\
      &=\sum_{k=d}^\infty
\Big(\frac\gamma k\Big)^d\det\!\big(\E[\A^\top\B\mid K\!=\!k]\big)
        \frac{\gamma^{k-d}\ee^{-\gamma}}{(k-d)!}.
    \end{align*}
    Note that $\frac\gamma k\,\E[\A^\top\B\mid K\!=\!k]=\E[\A^\top\B]$,
    which is independent of $k$. Thus we can rewrite the above
    expression as:
    \begin{align*}
\det\!\big(\E[\A^\top\B]\big)\sum_{k=d}^\infty\frac{\gamma^{k-d}\ee^{-\gamma}}{(k-d)!}
      =
      \det\!\big(\E[\A^\top\B]\big)\sum_{k=0}^\infty
      \frac{\gamma^{k}\ee^{-\gamma}}{k!}=\det\!\big(\E[\A^\top\B]\big),
    \end{align*}
    which concludes the proof.
  \end{proofof}

To prove Lemma \ref{l:normalization}, we use the following standard
determinantal formula which is used to derive the normalization
constant of a discrete determinantal point process.
\begin{lemma}[\cite{dpp-ml}]\label{l:det-standard}
  For any $k\times d$ matrices $\A,\B$ we have
  \[\det(\I+\A\B^\top)=\sum_{S\subseteq[k]}\det(\A_{S,*}\B_{S,*}^\top).\]
\end{lemma}

\begin{proofof}{Lemma}{\ref{l:normalization}}
By Lemma \ref{l:poisson}, the matrix $\B^\top\A$ is determinant
preserving. Applying Lemma \ref{t:ring} we conclude that
$\I+\B^\top\A$ is also d.p., so
\begin{align*}
  \det\!\big(\I+\E[\B^\top\A]\big) = \E\big[\det(\I+\B^\top\A)\big] =\E\big[\det(\I+\A\B^\top)\big],
\end{align*}
where the second equality is known as Sylvester's Theorem.
We rewrite the expectation of $\det(\I+\A\B^\top)$ by applying Lemma
\ref{l:det-standard}.  Letting $\gamma=\E[K]$, we obtain:
\begin{align*}
\E\big[\det(\I+\A\B^\top)\big]  &=\E\bigg[\sum_{S\subseteq [K]}\E\big[\det(\A_{S,*}\B_{S,*}^\top)\mid
    K\big]\bigg]\\
  &\overset{(*)}{=}\sum_{k=0}^\infty\frac{\gamma^k\ee^{-\gamma}}{k!}
  \sum_{i=0}^k\binom{k}{i} \E\big[\det(\A\B^\top)\mid K=i\big]\\
  &=\sum_{i=0}^\infty \E\big[\det(\A\B^\top)\mid K=i\big]
  \sum_{k\geq i}^\infty \binom{k}{i}
  \frac{\gamma^k\ee^{-\gamma}}{k!}\\
  &=\sum_{i=0}^\infty
    \frac{\gamma^i\ee^{-\gamma}}{i!}\E\big[\det(\A\B^\top)\mid K=i\big]
    \sum_{k\geq i}^\infty\frac{\gamma^{k-i}}{(k-i)!} = \E\big[\det(\A\B^\top)\big]\cdot\ee^\gamma,
\end{align*}
where $(*)$ follows from the exchangeability of the rows of $\A$ and
$\B$, which implies that the distribution of $\A_{S,*}\B_{S,*}^\top$ is the
same for all subsets $S$ of a fixed size $k$.
\end{proofof}

\section{Proof of Theorem \ref{t:mse}}
\label{a:mse-proof}
In this section we use $Z_\mu^n$ to denote the normalization
constant that appears in \eqref{eq:cases} when computing an expectation for surrogate design
$S_\mu^n$.
We first prove Lemma \ref{l:sqinv-all}. 
\begin{lemma}[restated Lemma \ref{l:sqinv-all}]\label{l:sqinv-under}
If  $\Xb\sim S_\mu^n$ for $n<d$, then we have
\begin{align*}
    \E\big[\tr\big((\Xb^\top\Xb)^{\dagger}\big)\big]
    &={\gamma_n}\big(1- \det\!\big((\tfrac1{\gamma_n}\I+\Sigmab_\mu)^{-1}\Sigmab_\mu\big)\big).
\end{align*}
\end{lemma}
\begin{proof}
Let $\X\sim\mu^K$ for $K\sim\Poisson({\gamma_n})$. Note that if
$\det(\X\X^\top)>0$ then using the fact that
$\det(\A)\A^{-1}=\adj(\A)$ for any invertible matrix $\A$, we can write:
  \begin{align*}
    \det(\X\X^\top)\tr\big((\X^\top\X)^{\dagger}\big)
    &= \det(\X\X^\top)\tr\big((\X\X^\top)^{-1}\big) \\
    &= \tr(\adj(\X\X^\top)) \\[-1mm]
    &= \sum_{i=1}^K\det(\X_{-i}\X_{-i}^\top),
  \end{align*}
  where $\X_{-i}$ is a shorthand for $\X_{[K]\backslash\{i\},*}$.
Assumption \ref{a:general-position} ensures that
$\Pr\big\{\det(\X\X^\top)>0\big\}=1$, which allows us to write:
  \begin{align*}
Z_\mu^n\cdot \E\big[\tr\big((\Xb^\top\Xb)^{\dagger}\big)\big]
    &=\E\bigg[
    \sum_{i=1}^K\det(\X_{-i}\X_{-i}^\top)\ \big|\
    \det(\X\X^\top)>0\bigg]\cdot\overbrace{\Pr\big\{\det(\X\X^\top)>0\big\}}^{1}\\
    &=\sum_{k=0}^d\frac{\gamma_n^{k}\ee^{-\gamma_n}}{k!}\E\Big[
      \sum_{i=1}^k\det(\X_{-i}\X_{-i}^\top)\ \big|\  K=k\Big]\\
    &=\sum_{k=0}^d\frac{\gamma_n^{k}\ee^{-\gamma_n}}{k!}\, k\
      \E\big[\det(\X\X^\top)\mid K=k-1\big]\\
    &=\gamma_n\sum_{k=0}^{d-1}\frac{\gamma_n^{k}\ee^{-\gamma_n}}{k!}
      \E\big[\det(\X\X^\top)\mid K=k\big]\\
    &=\gamma_n\Big(\E\big[\det(\X\X^\top)\big]\ -\
      \frac{\gamma_n^{d}\ee^{-\gamma_n}}{d!}\E\big[\det(\X)^2\mid K=d\big]
      \Big) \\
    &\overset{(*)}{=}\gamma_n\big(\ee^{-\gamma_n}\det(\I +\gamma_n\Sigmab_\mu) -
      \ee^{-\gamma_n}\det(\gamma_n\Sigmab_\mu)\big),
  \end{align*}
  where $(*)$ uses Lemma \ref{l:normalization} for the first term and
  Lemma \ref{l:cb} for the second term. We obtain the desired result by
  dividing both sides by
  $Z_\mu^n=\ee^{-\gamma_n}\det(\I+\gamma_n\Sigmab_\mu)$.
\end{proof}
In the over-determined regime, a more general matrix expectation
formula can be shown (omitting the trace). The following result is
related to an expectation formula derived by
\cite{correcting-bias-journal}, however they use a slightly
different determinantal design so the results are incomparable.
\begin{lemma}\label{l:sqinv-over}
If $\Xb\sim S_\mu^n$ and $n>d$, then we
have
\begin{align*}
  \E\big[ (\Xb^\top\Xb)^{\dagger}\big] =
  \Sigmab_\mu^{-1}\cdot \frac{1-\ee^{-\gamma_n}}{\gamma_n}.
\end{align*}
\end{lemma}
\begin{proof}
Let $\X\sim\mu^K$ for $K\sim\Poisson(\gamma_n)$. Assumption
\ref{a:general-position} implies that for $K\neq d-1$ we have
\begin{align}
  \det(\X^\top\X)(\X^\top\X)^\dagger=\adj(\X^\top\X),\label{eq:adj-over}
  \end{align}
however when $k=d-1$ then \eqref{eq:adj-over} does not hold because
$\det(\X^\top\X)=0$ while $\adj(\X^\top\X)$ may be non-zero. It
follows that:
  \begin{align*}
Z_\mu^n\cdot
    \E\big[ (\Xb^\top\Xb)^{\dagger}\big]
    &=\E\big[\det(\X^\top\X)(\X^\top\X)^\dagger\big]\\
    &=\E\big[\adj(\X^\top\X)\big]-
\frac{\gamma_n^{d-1}\ee^{-\gamma_n}}{(d-1)!}
      \E\big[\adj(\X^\top\X)\mid K=d-1\big]\\
    &\overset{(*)}{=}\adj\!\big(\E[\X^\top\X]\big) -
      \frac{\gamma_n^{d-1}\ee^{-\gamma_n}}{(d-1)^{d-1}}
      \adj\!\big(\E[\X^\top\X\mid K=d-1]\big)\\
    &=\adj(\gamma_n\Sigmab_\mu) - \ee^{-\gamma_n}\adj(\gamma_n\Sigmab_\mu)\\
    &=\det(\gamma_n\Sigmab_\mu)\,(\gamma_n\Sigmab_\mu)^{-1}(1-\ee^{-\gamma_n})\\
    &=\det(\gamma_n\Sigmab_\mu)\,\Sigmab_\mu^{-1}\cdot\frac{1-\ee^{-\gamma_n}}{\gamma_n},
  \end{align*}
  where the first term in $(*)$ follows from Lemma
  \ref{l:normalization} and \eqref{eq:adj}, whereas the second term comes
  from Lemma 2.3 of \cite{correcting-bias-journal}.
Dividing both sides by $Z_\mu^n=\det(\gamma_n\Sigmab_\mu)$ completes the proof.
\end{proof}

Applying the closed form expressions from Lemmas
\ref{l:proj}, \ref{l:sqinv-all} and \ref{l:sqinv-over}, we derive
the formula for the MSE and prove Theorem \ref{t:mse} (we defer the
proof of Lemma \ref{l:proj} to Appendix \ref{s:unbiased-proof}).
\begin{proofof}{Theorem}{\ref{t:mse}}
  First, assume that $n<d$, in which case we have
  $\gamma_n=\frac1{\lambda_n}$ and moreover
  \begin{align*}
    n &= \tr\big(\Sigmab_\mu(\Sigmab_\mu+\lambda_n\I)^{-1}\big)\\
      &=\tr\big((\Sigmab_\mu+\lambda_n\I-\lambda_n\I)(\Sigmab_\mu+\lambda_n\I)^{-1}\big)\\
    &=d - \lambda_n\tr\big((\Sigmab_\mu+\lambda_n\I)^{-1}\big),
  \end{align*}
so we can write $\lambda_n$ as $(d-n)/\tr((\Sigmab_\mu+\lambda_n\I)^{-1})$.
  From this and Lemmas \ref{l:proj} and \ref{l:sqinv-under}, we
obtain the desired expression, where recall
  that $\alpha_n = \det\!\big(\Sigmab_\mu (\Sigmab_\mu+\frac1{\gamma_n})^{-1}\big)$:
  \begin{align*}
    \MSE{\Xb^\dagger\ybb} &= \sigma^2\,\gamma_n(1-\alpha_n) +
    \tfrac1{\gamma_n} \,\w^{*\top}(\Sigmab_\mu+\tfrac1{\gamma_n}\I)^{-1}\w^*
    \\
    &\overset{(a)}{=}\sigma^2\,\frac{1-\alpha_n}{\lambda_n} +
    \lambda_n\,\w^{*\top}(\Sigmab_\mu+\lambda_n\I)^{-1}\w^*\\
    &\overset{(b)}{=}\sigma^2\tr\big((\Sigmab_\mu+\lambda_n\I)^{-1}\big)\frac{1-\alpha_n}{d-n}
      +
      (d-n)\frac{\w^{*\top}(\Sigmab_\mu+\lambda_n\I)^{-1}\w^*}
      {\tr\big((\Sigmab_\mu+\lambda_n\I)^{-1}\big)}.
  \end{align*}
  While the expression given after $(a)$ is simpler than the one
after $(b)$, the latter better illustrates how the MSE depends on
the sample size $n$ and the dimension $d$.
  Now, assume that $n>d$. In this case, we have $\gamma_n=n-d$ and apply Lemma
  \ref{l:sqinv-over}:
  \begin{align*}
    \MSE{\Xb^\dagger\ybb}
    &= \sigma^2\,\tr(\Sigmab_\mu^{-1})\,
\frac{1-\ee^{-\gamma_n}}{\gamma_n}
=\sigma^2\,\tr(\Sigmab_\mu^{-1})
\,\frac{1-\beta_n}{n-d}.
  \end{align*}
The case of $n=d$ was shown in Theorem~2.12 of \cite{correcting-bias-journal}.
This concludes the proof.
\end{proofof}

\section{Proof of Theorem \ref{t:unbiased}}
\label{s:unbiased-proof}

As in the previous section, we use $Z_\mu^n$ to denote the normalization
constant that appears in \eqref{eq:cases} when computing an expectation
for surrogate design $S_\mu^n$.
Recall that our goal is to compute the expected value of
$\Xb^\dagger\ybb$ under the surrogate design $S_\mu^n$. Similarly as for Theorem
\ref{t:mse}, the case of $n=d$ was shown in Theorem 2.10 of
\cite{correcting-bias-journal}. We break the rest down into the
under-determined case $(n<d)$ and the over-determined case ($n>d$),
starting with the former. Recall that we do \emph{not} require any
modeling assumptions on the responses.
\begin{lemma}\label{l:ridge-under}
If $\Xb\sim S_\mu^n$ and $n<d$, then for any $y(\cdot)$
such that $\E_{\mu,y}[y(\x)\,\x]$ is well-defined,
denoting $\yb_i$ as $y(\xbb_i)$, we have
\begin{align*}
  \E\big[\Xb^\dagger \ybb\big]
  &=
    \big(\Sigmab_\mu+\tfrac1{\gamma_n}\I\big)^{-1}\E_{\mu,y}[y(\x)\,\x].
\end{align*}
\end{lemma}
\begin{proof}
   Let $\X\sim\mu^K$ for $K\sim\Poisson(\gamma_n)$ and denote
   $y(\x_i)$ as $y_i$.
  Note that when $\det(\X\X^\top)>0$, then
  the $j$th entry of $\X^\dagger\y$ equals
  $\f_j^\top(\X\X^\top)^{-1}\y$, where $\f_j$ is the $j$th
  column of $\X$, so:
\begin{align*}
  \det(\X\X^\top)\,(\X^\dagger\y)_j
  &= \det(\X\X^\top)\, \f_j^\top(\X\X^\top)^{-1}\y \\
  &=
  \det(\X\X^\top+\y\f_j^\top) - \det(\X\X^\top).
\end{align*}
If $\det(\X\X^\top)=0$, then also
$\det(\X\X^\top+\y\f_j^\top)=0$, so we can write:
\begin{align*}
Z_\mu^n\cdot\E\big[(\Xb^\dagger\ybb)_j\big]
  &=  \E\big[\det(\X\X^\top)(\X^\dagger\y)_j\big] \\
  &= \E\big[\det(\X\X^\top+\y\f_j^\top)-\det(\X\X^\top)\big]  \\
  &=\E\big[\det\!\big([\X,\y][\X,\f_j]^\top\big)\big] - \E\big[\det(\X\X^\top)\big]\\
  &\overset{(a)}{=}\ee^{-\gamma_n}\det\!\bigg(\I +
    \gamma_n\,\E_{\mu,y}\bigg[\begin{pmatrix}\x\x^\top&
      \!\!\x\, y(\x)\\ x_j\,\x^\top&\!\!
      x_j\,y(\x)\end{pmatrix}\bigg]\bigg)
                          -\ee^{-\gamma_n}\det(\I+\gamma_n\Sigmab_\mu)\\
  &\overset{(b)}{=}\ee^{-\gamma_n}\det(\I+\gamma_n\Sigmab_\mu) \\
    &\qquad \times \Big(\E_{\mu,y}\big[\gamma_n
    x_j\,y(\x)\big] - \E_{\mu}\big[\gamma_n
    x_j\,\x^\top\big](\I+\gamma_n\Sigmab_\mu)^{-1}\E_{\mu,y}\big[\gamma_n\x\,
    y(\x)\big]\Big),
\end{align*}
where $(a)$ uses Lemma \ref{l:normalization} twice, with the first
application involving two different matrices $\A=[\X,\y]$ and
$\B=[\X,\f_j]$, whereas $(b)$ is a standard determinantal identity
\cite[see Fact 2.14.2 in][]{matrix-mathematics}.
  Dividing both sides by $Z_\mu^n$ and letting $\v_{\mu,y}=\E_{\mu,y}[y(\x)\,\x]$, we obtain that:
  \begin{align*}
    \E\big[\Xb^\dagger\ybb\big]
    &= \gamma_n\v_{\mu,y} - \gamma_n^2\Sigmab_\mu(\I+\gamma_n\Sigmab_\mu)^{-1}\v_{\mu,y}\\
&=\gamma_n\big(\I - \gamma_n\Sigmab_\mu (\I+\gamma_n\Sigmab_\mu)^{-1}\big)\v_{\mu,y}
=\gamma_n(\I+\gamma_n\Sigmab_\mu)^{-1}\v_{\mu,y},
  \end{align*}
  which completes the proof.
\end{proof}
We return to Lemma \ref{l:proj}, regarding the expected orthogonal
projection onto the complement of the row-span of $\Xb$, i.e.,
$\E[\I-\Xb^\dagger\Xb]$, which follows as a corollary of Lemma~\ref{l:ridge-under}.
\begin{proofof}{Lemma}{\ref{l:proj}}
  We let $y(\x)=x_j$ where $j\in[d]$ and apply Lemma
  \ref{l:ridge-under} for each $j$, obtaining:
  \begin{align*}
    \I - \E\big[\Xb^\dagger\Xb] = \I -
    (\Sigmab_\mu+\tfrac1{\gamma_n}\I)^{-1}\Sigmab_\mu,
  \end{align*}
  from which the result follows by simple algebraic manipulation.
\end{proofof}

We move on to the over-determined case, where the ridge regularization
of adding the identity to $\Sigmab_\mu$ vanishes. Recall that we
assume throughout the paper that $\Sigmab_\mu$ is invertible.
\begin{lemma}\label{l:ridge-over}
  If $\Xb\sim S_\mu^n$ and $n>d$, then for any real-valued random function $y(\cdot)$
  such that $\E_{\mu,y}[y(\x)\,\x]$ is well-defined,
denoting $\yb_i$ as $y(\xbb_i)$, we have
 \begin{align*}
  \E\big[\Xb^\dagger \ybb\big]
  &=\Sigmab_\mu^{-1}\E_{\mu,y}\big[y(\x)\,\x\big].
\end{align*}
\end{lemma}
\begin{proof}
   Let $\X\sim\mu^K$ for $K\sim\Poisson(\gamma_n)$ and denote
   $y_i=y(\x_i)$. Similarly as in the proof of
   Lemma~\ref{l:ridge-under}, we note that when $\det(\X^\top\X)>0$,
   then
  the $j$th entry of $\X^\dagger\y$ equals
  $\e_j^\top(\X^\top\X)^{-1}\X^\top\y$, where $\e_j$ is the $j$th
standard basis vector, so:
\begin{align*}
  \det(\X^\top\X)\,(\X^\dagger\y)_j =
  \det(\X^\top\X)\, \e_j^\top(\X^\top\X)^{-1}\X^\top\y =
  \det(\X^\top\X+\X^\top\y\e_j^\top) - \det(\X^\top\X).
\end{align*}
If $\det(\X^\top\X)=0$, then also
$\det(\X^\top\X+\X^\top\y\e_j^\top)=0$. We proceed to compute the
expectation:
\begin{align*}
Z_\mu^n\cdot\E\big[(\Xb^\dagger\ybb)_j\big]
  &=  \E\big[\det(\X^\top\X)(\X^\dagger\y)_j\big] \\
  &= \E\big[\det(\X^\top\X+\X^\top\y\e_j^\top)-\det(\X^\top\X)\big]  \\
  &=\E\big[\det\!\big(\X^\top(\X+\y\e_j^\top)\big)\big] - \E\big[\det(\X^\top\X)\big]\\
  &\overset{(*)}{=}\det\!\Big(
    \gamma_n\,\E_{\mu,y}\big[\x(\x+ y(\x)\e_j)^\top\big]\Big)
    -\det(\gamma_n\Sigmab_\mu)\\
  &=\det\!\big(\gamma_n\Sigmab_\mu + \gamma_n\E_{\mu,y}[\x\,y(\x)]\e_j^\top\big)
    -\det(\gamma_n\Sigmab_\mu)\\
  &=\det(\gamma_n\Sigmab_\mu)\cdot
    \gamma_n\e_j^\top(\gamma_n\Sigmab_\mu)^{-1}\E_{\mu,y}\big[y(\x)\,\x\big],
\end{align*}
where $(*)$ uses Lemma \ref{l:poisson} twice (the first time, with
$\A=\X$ and $\B=\X+\y\e_j^\top$). Dividing both sides by
$Z_\mu^n=\det(\gamma_n\Sigmab_\mu)$ concludes the proof.
\end{proof}

\noindent
We combine Lemmas \ref{l:ridge-under} and \ref{l:ridge-over} to obtain
the proof of Theorem \ref{t:unbiased}.
\begin{proofof}{Theorem}{\ref{t:unbiased}}
The case of $n=d$ follows directly from Theorem~2.10 of
\cite{correcting-bias}.
Assume that $n<d$. Then we have
$\gamma_n=\frac1{\lambda_n}$, so the result follows
from Lemma \ref{l:ridge-under}.
If $n>d$, then the result follows from Lemma \ref{l:ridge-over}.
\end{proofof}

\section{Proof of Theorem \ref{t:asymptotic}}
\label{sec:proof-of-t-asymptotic}

The proof of Theorem \ref{t:asymptotic} follows the standard decomposition of
MSE in Equation~\ref{eq:mse-derivation}, and in the process,
establishes consistency of the variance and bias terms
independently. To this end, we introduce the following two useful
lemmas that capture the limiting behavior of the 
variance and bias terms, respectively.

\begin{lemma}\label{c:wishart}
  Under the setting of Theorem~\ref{t:asymptotic}, we have, as $n,d \to \infty$
with $n/d \to \bar c \in (0, \infty) \setminus \{ 1 \}$ that
\begin{align}
\begin{cases}
  \E\big[\tr((\X^\top \X)^\dagger)\big] - (1-\alpha_n) \lambda_n^{-1} \to 0, & \text{for }\bar c < 1, \\
  \E\big[\tr((\X^\top \X)^\dagger)\big] - \frac{1-\beta_n}{n - d} \cdot \tr \Sigmab^{-1} \to 0, & \text{for }\bar c > 1
\end{cases}
\label{eq:wishart}
\end{align}
where $\lambda_n\geq 0$ is the unique solution to
$n=\tr(\Sigmab(\Sigmab+\lambda_n\I)^{-1})$,
$\alpha_n=\det(\Sigmab(\Sigmab+\lambda_n\I)^{-1})$,
and $\beta_n = e^{d-n}$.
\end{lemma}

\noindent

The second term in the MSE derivation \eqref{eq:mse-derivation}, $\E[\I -
\X^\dag \X]$, involves the expectation of a projection onto the orthogonal
complement of a sub-Gaussian general position sample $\X$.
This term is zero when $n > d$, and for $n < d$
we prove in \cref{sec:proof-of-c-projection} that the surrogate design's
bias $\mathcal{B}(\Sigmab, n)$ provides an asymptotically consistent approximation
to all of the eigenvectors and eigenvalues:

\begin{lemma}\label{c:projection}
Under the setting of Theorem~\ref{t:asymptotic}, for $\w \in \R^d$ of bounded Euclidean norm (i.e., $\| \w \| \le C'$ for all $d$), we have, as $n,d \to \infty$ with $n/d \to \bar c \in (0,1)$ that
\begin{align}
  \w^\top\E[\I-\X^\dagger\X]\w - \lambda_n \w^\top (\Sigmab + \lambda_n \I)^{-1}\w \to 0
  \label{eq:projection}
\end{align}
while $\I - \X^\dagger \X = 0$ for $\bar c > 1$.
\end{lemma}

\subsection{Proof of \cref{c:wishart}}
\label{sec:proof-of-c-wishart}

\subsubsection{The $\bar c \in (0,1)$ case}

For $n < d$, we first establish (1) $\liminf_n \lambda_n > 0$ and (2) $\alpha_n \to 0$.
To prove (1), by hypothesis $\Sigmab \succeq c \I$ for all $d$. Since $\frac{n}{d} < 1$,
we have (by definition of $\lambda_n$) for some $\delta > 0$
\begin{align*}
  1 - \delta
  > \frac{n}{d}
  = \frac{1}{d} \tr(\Sigmab(\Sigmab + \lambda_n \I)^{-1})
  > \frac{c}{c + \lambda_n}
\end{align*}
Rearranging, we have $\lambda_n > \frac{\delta c}{1 - \delta} > 0$.
For (2), let $(\tau_i)_{i \in [d]}$ denote the eigenvalues of $\Sigmab$.
Since $1 - x \leq e^{-x}$ and $C \I \succeq \Sigmab \succeq c \I$ for all $d$,
\begin{align*}
  \alpha_n
  = \prod_{i=1}^d \frac{\tau_i}{\tau_i + \lambda_n}
  \leq \left(\frac{C}{C + \lambda_n}\right)^d
  = \left( 1 - \frac{\lambda_n}{C + \lambda_n}\right)^d
  \leq \exp\left(-d \frac{\lambda_n}{C + \lambda_n}\right)
\end{align*}
and since $\lambda_n > 0$ eventually as $d \to \infty$ we have
$\alpha_n \to 0$ so that $(1-\alpha_n) \lambda_n^{-1} - \lambda_n^{-1} \to 0$.

As a consequence of (2) and Slutsky's theorem, it suffices to show
$\tr(\X^\top \X)^\dag - \lambda_n^{-1} \overset{d}{\to} 0$ as $n,d \to \infty$.
To do this, we consider the limiting behavior of $\tr
(\X^\top \X)^\dagger/n = \tr (\X \X^\top)^\dagger/n$ as $n/d \to \bar c \in
(0,1)$, for $\X = \Z \Sigmab^{\frac12}$ with $\Z \in \mathbb R^{n \times d}$
having i.i.d.~zero mean, unit variance sub-Gaussian entries, i.e., the behavior
of
\begin{equation}\label{eq:limit1}
  \lim_{n,d \to \infty} \lim_{z \to 0^+} \frac1n \tr \left( \frac1n \X \X^\top + z \I_n \right)^{-1}
\end{equation}
by definition of the pseudo-inverse.

The proof comes in three steps: (i) for fixed $z > 0$, consider the limiting behavior of $ \delta(z) \equiv \tr (\X \X^\top/n + z \I_n)^{-1}/n$ as $n,d \to \infty$ and state
\begin{equation}\label{eq:limit2}
  \lim_{n,d \to \infty} \delta(z) - m(z) \to 0
\end{equation}
almost surely for some $m(z)$ to be defined; (ii) show that both $\delta(z)$ and its derivate $\delta'(z)$ are uniformly bounded (by some quantity independent of $z>0$) so that by Arzela-Ascoli theorem, $\delta(z)$ converges uniformly to its limit and we are allowed to take $z \to 0^+$ in \eqref{eq:limit2} and state
\begin{equation}\label{eq:limit3}
  \lim_{z \to 0^+} \lim_{n,d \to \infty} \delta(z) - \lim_{z \to 0^+} m(z) \to 0
\end{equation}
almost surely, given that the limit $\lim_{z \to 0^+} m(z) \equiv m(0)$ exists and eventually (iii) exchange the two limits in \eqref{eq:limit3} with Moore-Osgood theorem, to reach
\[
  \lim_{n,d \to \infty} \lim_{z \to 0^+} \frac1n \tr \left(\frac1n \X \X^\top + z \I_n \right)^{-1} - m(0) \to 0.
\]

Step (i) follows from \cite{silverstein1995empirical} that, we have, for $z > 0$ that
\[
  \delta(z) \equiv \frac1n \tr \left( \frac1n \X \X^\top  + z \I_n \right)^{-1}  - m(z) \to 0
\]
almost surely as $n,d \to \infty$, for $m(z)$ the unique positive solution to
\begin{equation}\label{eq:def-m}
  m(z) = \left( z + \frac1n \tr \Sigmab (\I + m(z) \Sigmab)^{-1} \right)^{-1}.
\end{equation}

For the above step (ii), we use the assumption $\Sigmab \succeq c \I \succ 0$
for all $d$ large, so that with $\X = \Z \Sigmab^{\frac12}$, we have for large enough $n,d$ that
\[
  \lambda_{\min} (\X \X^\top/n) \ge \lambda_{\min} (\Z \Z^\top/n) \lambda_{\min} (\Sigmab) \ge \frac{c}2 (\sqrt{\bar c} - 1 )^2
\]
almost surely, where we used Bai-Yin theorem \cite{bai1993limit}, which states that the minimum eigenvalue of $\Z \Z^\top/n$ is almost surely larger than $( \sqrt{\bar c} - 1 )^2/2$ for $n<d$ sufficiently large. Note that here the case $\bar c = 1$ is excluded.

Observe that
\[
  |\delta(z)| = \left| \frac1n \tr \left( \frac1n \X \X^\top  + z \I_n \right)^{-1} \right| \le \frac1{ \lambda_{\min} (\X \X^\top/n) }
\]
and similarly for its derivative, so that we are allowed to take the $z \to 0^+$ limit. Note that the existence of the $\lim_{z \to 0^+} m(z)$ for $m(z)$ defined in \eqref{eq:def-m} is well known, see for example \cite{ledoit2011eigenvectors}. Then, by Moore-Osgood theorem we finish step (iii) and by concluding that
\[
  \tr (\X^\top \X)^\dagger - m(0) \to 0
\]
for $m(0) = \lambda_n^{-1}$ the unique solution to $\lambda_n^{-1} = \left( \frac1n \tr \Sigmab (\I + \lambda_n^{-1} \Sigmab)^{-1} \right)^{-1}$, or equivalently, to
\[
  n = \tr \Sigmab (\Sigmab + \lambda_n \I)^{-1}
\]
as desired.

\subsubsection{The $\bar c \in (1, \infty)$ case}

First note that as $n, d \to \infty$ with $n > d$, we have $\beta_n = e^{d-n} \to 0$ and it
it suffices to show
\[
  \tr (\X^\top \X)^\dagger - \frac1{n-d} \tr \Sigmab^{-1} \to 0
\]
almost surely to conclude the proof.

In the $\bar c \in (1, \infty)$ case, it is more convenient to work on the following co-resolvent
\[
  \lim_{n,d \to \infty} \lim_{z \to 0^+} \frac1n \tr \left( \frac1n \X^\top \X + z \I_d \right)^{-1}
\]
where we recall $\X^\top \X = \Sigmab^{\frac12} \Z^\top \Z \Sigmab^{\frac12} \in \R^{d \times d}$ and following the same three-step procedure as in the $\bar c < 1$ case above. The only difference is in step (i) we need to assess the asymptotic behavior of $\delta \equiv \tr (\X^\top \X/n + z \I_d)^{-1}/n$. This was established in \cite{bai1998no} where it was shown that, for $z > 0$ we have
\[
  \frac1n \tr (\X^\top \X/n + z \I_d)^{-1} - \frac{d}n m(z) \to 0
\]
almost surely as $n,d \to \infty$, for $m(z)$ the unique solution to
\[
  m(z) = \frac1d \tr \left( \left(1- \frac{d}n - \frac{d}n z m(z) \right) \Sigmab - z \I_d \right)^{-1}
\]
so that for $d < n$ by taking $z = 0$ we have
\[
  m(0) = \frac{n}d \frac1{n-d} \tr \Sigmab^{-1}.
\]
The steps (ii) and (iii) follow exactly the same line of arguments as the $\bar c < 1$ case and are thus omitted.

\subsection{Proof of \cref{c:projection}}
\label{sec:proof-of-c-projection}

Since $\X^\dagger \X = \X^\top (\X \X^\top)^\dagger \X$, to prove
\cref{c:projection}, we are interested in the limiting behavior of the
following quadratic form
\[
  \lim_{n,d \to \infty} \lim_{z \to 0^+} \frac1n \w^\top \X^\top \left( \frac1n \X \X^\top + z \I_n \right)^{-1} \X \w
\]
for deterministic $\w \in \mathbb R^{d}$ of bounded Euclidean norm (i.e., $\| \w \| \le C'$ as $n,d \to \infty$), as $n,d \to \infty$ with $n/d \to \bar c \in (0,1)$. The limiting behavior of the above quadratic form, or more generally, bilinear form of the type $\frac1n \w_1^\top \X^\top \left( \frac1n \X \X^\top + z \I_n \right)^{-1} \X \w_2$ for $\w_1, \w_2 \in \mathbb R^{d}$ of bounded Euclidean norm are widely studied in random matrix literature, see for example \cite{hachem2013bilinear}.

For the proof of \Cref{c:projection} we follow the same protocol as that of
\Cref{c:wishart}, namely: (i) we consider, for fixed $z > 0$, the limiting
behavior of $\frac1n \w^\top \X^\top \left( \frac1n \X \X^\top + z \I_n
\right)^{-1} \X \w$. Note that
\begin{align*}
  \delta(z) &\equiv \frac1n \w^\top \X^\top \left( \frac1n \X \X^\top + z \I_n \right)^{-1} \X \w = \w^\top \left( \frac1n \X^\top \X + z \I_d \right)^{-1} \frac1n \X^\top \X \w \\
  &= \| \w \|^2 - z \w^\top \left( \frac1n \X^\top \X + z \I_d \right)^{-1} \w
\end{align*}
and it remains to work on the second $z \w^\top \left( \frac1n \X^\top \X + z \I_d \right)^{-1} \w$ term.
It follows from \cite{hachem2013bilinear} that
\[
   z \w^\top \left( \frac1n \X^\top \X + z \I_d \right)^{-1} \w - \w^\top (\I_d + m(z) \Sigmab)^{-1} \w^\top \to 0
\]
almost surely as $n,d \to \infty$, where we recall $m(z)$ is the unique solution to \eqref{eq:def-m}.

We move on to step (ii), under the assumption that $c \le \lambda_{\min} (\Sigmab) \le \lambda_{\max} (\Sigmab) \le C$ and $\| \w \| \le C'$, we have
\begin{align*}
  \lambda_{\max} \left( \frac1n \X^\top \left( \frac1n \X \X^\top + z \I_n \right)^{-1} \X \right) &\le \frac{ \lambda_{\max} (\X \X^\top/n) }{ \lambda_{\min} (\X \X^\top/n) + z } \le \frac{ \lambda_{\max} (\Z \Z^\top/n) \lambda_{\max} (\Sigmab) }{ \lambda_{\min} (\Z \Z^\top/n) \lambda_{\min} (\Sigmab) } \\
  &\le 4 \frac{ (\sqrt{\bar c} + 1)^2 C }{ (\sqrt{\bar c} -1)^2 c }
\end{align*}
so that $\delta(z)$ remains bounded and similarly for its derivative
$\delta'(z)$, which, by Arzela-Ascoli theorem, yields uniform
convergence and we are allowed to take the $z \to 0^+$
limit. Ultimately, in step (iii) we exchange the two limits with
Moore-Osgood theorem, concluding the proof.

\subsection{Finishing the proof of Theorem~\ref{t:asymptotic}}

To finish the proof of Theorem~\ref{t:asymptotic}, it remains to write
\[
  \MSE{\X^\dagger\y} = \sigma^2\E\big[\tr\big((\X^\top\X)^{\dagger}\big)\big] +
    \w^{*\top}\E\big[\I-\X^\dagger\X\big]\w^*
\]
Since $\lambda_n = \frac{d-n}{ \tr (\Sigmab + \lambda_n \I)^{-1} }$,
by \Cref{c:wishart} and \Cref{c:projection} we have
$\MSE{\X^\dagger\y} - \Mc(\Sigmab, \w^*, \sigma^2, n) \to 0$ as $n,d
\to \infty$ with $n/d \to \bar c \in (0,\infty) \setminus \{ 1 \}$,
which concludes the proof of Theorem~\ref{t:asymptotic}.

\section{Additional details for empirical evaluation}
\label{a:empirical}

Our empirical investigation of the rate of asymptotic convergence in \Cref{t:asymptotic} (and, more
specifically, the variance and bias discrepancies defined in
Section~\ref{sec:asymp-conj-details}), in the context of
Gaussian random matrices, is related to open problems which have been
extensively studied in the literature. Note that when
$\X=\Z\Sigmab^{1/2}$ were $\Z$ has i.i.d.~Gaussian entries (as in
Section \ref{sec:asymp-conj-details}), then $\W=\X^\top\X$ is known as the
pseudo-Wishart distribution (also called the singular Wishart),
denoted as $\W \sim \Pc\Wc(\Sigmab, n)$, and the variance
term from the MSE can be written as $\sigma^2\E[\tr(\W^\dagger)]$.
\cite{srivastava2003} first derived the probability density function of the
pseudo-Wishart distribution, and
\cite{cook2011} computed the first and second moments of generalized
inverses. However, for the Moore-Penrose inverse and arbitrary 
covariance $\Sigmab$, \cite{cook2011} claims that the quantities required to
express the mean ``do not have tractable closed-form representation.''
The bias term, $\w^{*\top}\E[\I-\X^\dagger\X]\w^*$, has connections to
directional statistics.  Using the SVD, 
we have the equivalent representation $\X^\dagger \X = \V \V^\top$ where $\V$
is an element of the Stiefel manifold $V_{n,d}$ (i.e., orthonormal $n$-frames
in $\R^d$).  The distribution of $\V$ is known as the matrix angular central
Gaussian (MACG) distribution \citep{chikuse1990matrix}. While prior work has
considered high dimensional limit theorems \citep{CHIKUSE1991145} as well as
density estimation and hypothesis testing \citep{CHIKUSE1998188} on $V_{n,d}$,
they only analyzed the invariant measure (which corresponds in our setting to
$\Sigmab = \I$), and to our knowledge a closed form expression of
$\E[\V\V^\top]$ where $\V$ is distributed according to MACG with arbitrary
$\Sigmab$ remains an open question.

For analyzing the rate of decay of variance and bias discrepancies (as
defined in Section \ref{sec:asymp-conj-details}), it suffices to only consider diagonal
covariance matrices $\Sigmab$.  This is because if $\Sigmab = \Q \D \Q^\top$ is
its eigendecomposition and $\X\sim\Nc_{n,d}(\zero, \I_n \otimes \Q\D\Q^\top)$,
then we have for $\W \sim \Pc\Wc(\Sigmab, n)$ that $\W \overset{d}{=} \X^\top
\X$ and hence, defining $\Xt\sim\Nc_{n,d}(\zero,\I_n \otimes \D)$, by linearity
and unitary invariance of trace,
\begin{align*}
  \E[\tr(\W^\dagger)]
  &= \tr\big( \E[(\X^\top\X)^\dagger] \big)
  = \tr\Big( \Q\E\big[(\Xt^\top\Xt)^\dagger\big]\Q^\top \Big)
  = \tr\Big( \E\big[(\Xt^\top\Xt)^\dagger\big] \Big)
  = \E\left[\tr \big((\Xt^\top\Xt)^\dagger\big) \right].
\end{align*}
Similarly, we have that $\E[\X^\dagger\X]=\Q\E\big[\Xt^\dagger\Xt\big]\Q^\top$,
and a simple calculation shows that the bias discrepancy is
also independent of the choice of matrix $\Q$.

In our experiments, we increase $d$ while keeping the aspect ratio $n/d$
fixed and examining the rate of decay of the discrepancies.
We estimate $\E\big[\tr(\W^\dagger)\big]$ (for the variance) and
$\E[\I-\X^\dagger\X]$ (for the bias) through Monte Carlo sampling.
Confidence intervals are constructed using ordinary bootstrapping for
the variance. We rewrite the supremum over $\w$ in bias discrepancy as
a spectral norm: 
\[\big\|\Bc(\Sigmab,n)^{-\frac12}\E[\I-\X^\dagger\X]\Bc(\Sigmab,n)^{-\frac12} -
  \I\big\|,\]
and apply existing methods for constructing bootstrapped operator
norm confidence intervals described in \cite{lopes2019bootstrapping}.  To
ensure that estimation noise is sufficiently small, we continually increase the
number of Monte Carlo samples until the bootstrap confidence intervals are
within $\pm 12.5\%$ of the measured discrepancies.  We found that while
variance discrepancy required a relatively small number of trials (up
to one thousand), estimation noise was much larger for the bias
discrepancy, and it necessitated over two million trials to obtain
good estimates near $d=100$. 

\subsection{Eigenvalue decay profiles}
\label{sec:eig-decay-details}

Letting $\lambda_i(\Sigmab)$ be the $i$th largest eigenvalue of
$\Sigmab$, we consider the
following eigenvalue profiles (visualized in Figure~\ref{fig:eig-decays}):
\begin{itemize}
  \item \texttt{diag\_linear}: linear decay, $\lambda_i(\Sigmab)= b-a i$;
  \item \texttt{diag\_exp}: exponential decay, $\lambda_i(\Sigmab) = b\,10^{- a i} $;
  \item \texttt{diag\_poly}: fixed-degree polynomial decay, $\lambda_i(\Sigmab) = (b-a i)^2$;
  \item \texttt{diag\_poly\_2}: variable-degree polynomial decay, $\lambda_i(\Sigmab) = b i^{-a}$.
\end{itemize}
The constants $a$ and $b$ are chosen to ensure $\lambda_{\text{max}}(\Sigmab) = 1$ and
$\lambda_{\text{min}}(\Sigmab) = 10^{-4}$ (i.e., the condition number
$\kappa(\Sigmab) = 10^{4}$ remains constant).

\end{document}